%% file: 0_main.tex
\documentclass[letterpaper, conference, 10pt]{ieeeconf}  
\IEEEoverridecommandlockouts
\usepackage{graphicx}
\usepackage{epstopdf}
\usepackage{amsmath}
\usepackage{amssymb}
\usepackage{subfigure}
\usepackage{multirow}
\usepackage{pbox}
\usepackage{cite}
\usepackage{algorithm}
\usepackage[noend]{algpseudocode}
\usepackage{booktabs}
\usepackage{wrapfig}
\usepackage{lscape}
\usepackage{bm}
\usepackage{url}
\usepackage{esvect}
\usepackage{xcolor,soul}
\usepackage{txfonts}
\usepackage[normalem]{ulem}
\usepackage{tikz}
\usetikzlibrary{matrix,calc}
\usepackage{mathtools}

\DeclareSymbolFont{extraup}{U}{zavm}{m}{n}
\DeclareMathSymbol{\varheart}{\mathalpha}{extraup}{86}
\DeclareMathSymbol{\vardiamond}{\mathalpha}{extraup}{87}

\DeclarePairedDelimiterX{\infdivx}[2]{ }{ }{%
  #1\;\delimsize\|\;#2%
}

\newcommand{\R}{\mathbb{R}}

\DeclareMathOperator*{\argmin}{arg\,min}
\algnewcommand\algorithmicforeach{\textbf{for each}}
\algdef{S}[FOR]{ForEach}[1]{\algorithmicforeach\ #1\ \algorithmicdo}

\newcommand\AP{A\!P}

\newtheorem{problem}{Problem}
\newtheorem{lemma}{Lemma}

\renewcommand{\baselinestretch}{1}

\begin{document}

\title{\LARGE \bf Implicit Coordination using Active Epistemic Inference\\
for Multi-Robot Systems}
\author{Lauren Bramblett, Jonathan Reasoner, and Nicola Bezzo
\thanks{Lauren Bramblett, Jonathan Reasoner, and Nicola Bezzo are with the Departments of Systems and Information Engineering and Electrical and Computer Engineering, University of Virginia, Charlottesville, VA 22904, USA. 
Email: {\tt \{qbr5kx, vqh7rx, nb6be\}@virginia.edu}}
}

\maketitle

\input{_1_abstract}
\input{1_introduction}
\input{2_related_work}

\input{3_preliminaries}
\input{4_problem_formulation}

\input{5_results}

\input{6_conclusion}

\input{7_acks}

\bibliographystyle{IEEEtran}
\bibliography{sample.bib}

\end{document}

%% file: _1_abstract.tex
\begin{abstract}
A Multi-robot system (MRS) provides significant advantages for intricate tasks such as environmental monitoring, underwater inspections, and space missions. However, addressing potential communication failures or the lack of communication infrastructure in these fields remains a challenge. A significant portion of MRS research presumes that the system can maintain communication with proximity constraints, but this approach does not solve situations where communication is either non-existent, unreliable, or poses a security risk. Some approaches tackle this issue using predictions about other robots while not communicating, but these methods generally only permit agents to utilize first-order reasoning, which involves reasoning based purely on their own observations. In contrast, to deal with this problem, our proposed framework utilizes Theory of Mind (ToM), employing higher-order reasoning by shifting a robot's perspective to reason about a belief of others observations. Our approach has two main phases: i) an efficient runtime plan adaptation using active inference to signal intentions and reason about a robot's own belief and the beliefs of others in the system, and ii) a hierarchical epistemic planning framework to iteratively reason about the current MRS mission state. The proposed framework outperforms greedy and first-order reasoning approaches and is validated using simulations and experiments with heterogeneous robotic systems. 

\vspace{3pt}
\noindent\emph{Note---}Videos are provided in the supplementary material and also at {\url{https://www.bezzorobotics.com/lb-smcs24}}. Code will be available on the same website upon publication.
\end{abstract}

%% file: 1_introduction.tex
\section{Introduction}\label{sec:intro}
 Multi-robot systems (MRS) have the potential to transform current robotics applications, performing tasks more effectively and efficiently than a single robot. Central to MRS research is the idea that robots work together to accomplish common goals. The need for cooperative teaming is evident in numerous applications, such as search and rescue missions, firefighting, and underwater exploration. Effective collaboration requires robots to communicate and synchronize their actions, but difficulties often occur when communication is restricted, disrupted, or presents a security concern. Especially for a heterogeneous MRS where different robots have different operating or sensing capabilities. Humans have an inherent ability to ``see things from another's perspective" by understanding and sharing the beliefs of others without communicating. Imagine a parent attempting to convey a message to a child solely through their movements. The parent might understand the child's short attention span and limited observational skills, thus making exaggerated movements to communicate their intentions. In standard multi-robot missions, there are usually no strategies in place for communication breakdowns, or the strategies that do exist rely solely on each robot's first-order understanding of the environment and system. Planning actions socially requires a robot to infer the \textit{intentions} and \textit{ beliefs} of other agents, \textit{empathizing} to predict what other agents \textit{want} and \textit{know} about each other. The capacity to reason about the perspective of another agent is the foundation of theory of mind (ToM) which enables the {\em ``I know that you know that I know"} paradigm without the need for explicit communication among actors. 
 Formally, as described in \cite{valle2015theory}, there are different orders of reasoning possible: {\em zero}-order is a belief about oneself, {\em first}-order is a belief about others, {\em second}-order is a belief about what others believe about oneself, and {\em third}-order is a belief about what others believe about each other.
By using a second and higher order reasoning architecture, we can increase the operational effectiveness of multi-robot systems during disconnected operations, allowing robots to reason about the capability of other robots and plan according to local observations and distributed beliefs. Epistemic planning is an approach that integrates each agent's belief into a planning paradigm, enabling them to plan based not only on their perceptions but also on what they believe other agents know or intend to do~\cite{bolander2011epistemic}.
 
 \begin{figure}[t!]
\centering
\includegraphics[width=0.47\textwidth]{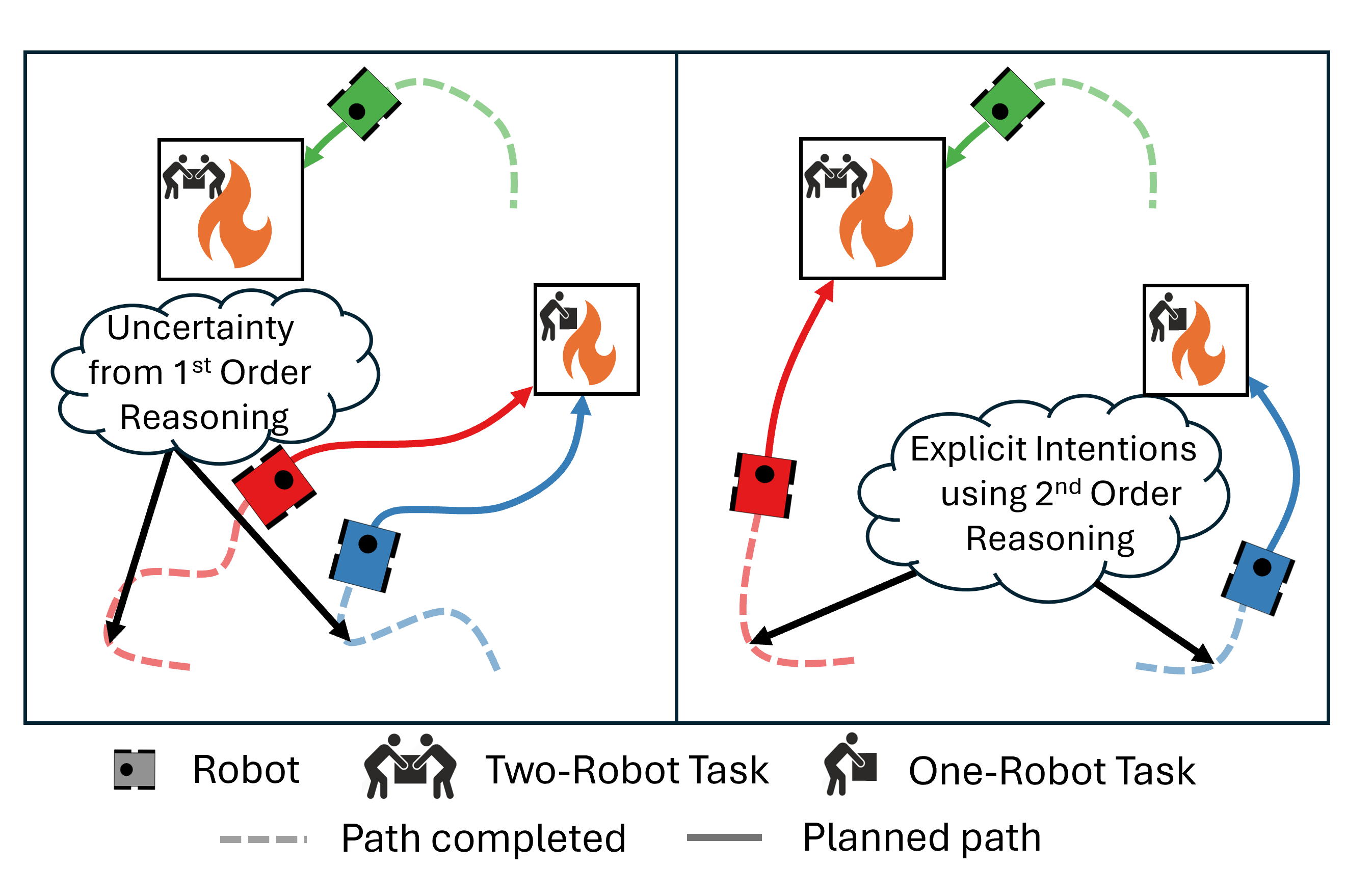}
\vspace{-10pt}
\caption{Pictorial representation of the problem presented in the paper. The robots are unable to explicitly communicate their beliefs and must convey their intentions through sensorimotor communication. In the left frame, the red and blue robot are unable to converge to a correct belief state using only first-order reasoning. In the right frame, the red and blue robot clearly display their intentions by using higher-order reasoning about their observations.}
\label{fig:intro_pic_aif}
\vspace{-20pt}
\end{figure}

 Theory of mind and epistemic planning can enable robots to reason about the probable knowledge and intentions of others. Our previous work shows that in environments with limited communication and uncertain operating conditions, epistemic planning can enable an MRS to achieve exploration and task allocation missions~\cite{bramblett2023epi,bramblett2023frontiers}. In conjunction, active inference can be used to compute belief-action pairs. Active inference operates on the principle that all entities strive to minimize variational free energy. This results in straightforward update rules for actions, perceptions, policy choices, learning processes, and the representation of uncertainty, which can extend to multi-agent processes \cite{maisto2023interactive}.  
 
 In this work, we focus on the following question: \textit{How can we ensure cooperative and efficient behavior for multi-robot tasks when robots cannot explicitly communicate?} Our proposed solution has two main components: 1) an efficient online planning mechanism that leverages active inference to signal to others their own knowledge and intentions based on the current epistemic state and probable goals to accomplish the mission, and 2) heirarchical epistemic planning that leverages our recent research~\cite{bramblett2024robust} which allows the robots to reason about the system goals and adapt its belief about the system at runtime.


Consider the example in Fig.~\ref{fig:intro_pic_aif}, where three robots cannot communicate and the mission is defined by two tasks, one requiring two robots and the other requiring only one robot. During the operation, each robot maintains a belief about the system defined by the likelihood that any robot is moving toward one of the two tasks in the environment. In the left frame, each robot uses only its first-order observations, reasoning about each robot based only on its own observations. In the right frame, each robot uses higher-order reasoning to not only infer other robots' goals based on its own observations but also by empathizing with how other robots' beliefs would change based on its own actions and their subsequent observations. We note that by using only first-order reasoning, the red and blue robots cannot determine the goals of the other, and they converge to an incorrect belief. Using higher-order reasoning, the red and blue robots both account for the observation of the other in deciding their next actions and clearly indicate their intent and beliefs. In this way, the MRS is able to converge to a correct belief state using only local observations. In this work, we utilize up to the third level of reasoning to allow robots to abstract the perceptions of other agents and more accurately derive the beliefs of the system. 

The contribution of this work is two-fold: i) a higher-order active inference framework for multi-robot task allocation without communication, and ii) an epistemic planning framework for belief updates and iterative task allocation. 
To the best of our knowledge, this is the first paper combining epistemic logic and active inference with runtime task allocation adaptations and no communication. We show that our higher-order reasoning method outperforms traditional greedy task allocation and first-order active inference algorithms to allocate tasks in an environment without communication.

The rest of the paper is organized as follows: in Section~\ref{sec:relWork}, we provide an overview of current research in multi-robot task allocation, epistemic planning, and active inference. In Section~\ref{sec:prob}, we formally define the problem followed by the framework for epistemic planning, active inference for decision-making and task allocation in Section~\ref{sec:approach}. Simulations and experiments validating our method are presented in Sections~\ref{sec:sims} and~\ref{sec:experiments}, respectively. Finally, conclusions and future work are discussed in Section~\ref{sec:conclusion}.

%% file: 2_related_work.tex
\section{Related Work}\label{sec:relWork}

Theory of Mind (ToM) and epistemic planning are closely related concepts in artificial intelligence and cognitive science~\cite{ho2022planning}. ToM refers to the ability of an agent to attribute mental states, such as beliefs, desires, and intentions to others, enabling it to predict and interpret their actions~\cite{anderson2004integrated} while epistemic planning is a type of automated planning in artificial intelligence that deals with knowledge and beliefs of agents~\cite{bolander2011epistemic}. Integrating ToM into epistemic planning allows agents to anticipate and respond to the knowledge and beliefs of other agents, leading to more effective coordination and decision-making in multi-agent systems. The authors in~\cite{muise2022efficient} and~\cite{zhang2023adaptation} show that nested beliefs and reasoning in multi-agent planning can better equip agents to work in teams and show that this integration is crucial for applications requiring sophisticated interaction and collaboration among multiple intelligent agents. This paper characterizes ToM for a multi-robot system similar to~\cite{engesser2017cooperative} in that we employ epistemic planning as a logical mechanism to account for the system's knowledge and beliefs. Epistemic planning can adopt the perspectives of other robots within the system, reasoning about their knowledge and uncertainties, thereby preventing first-order reasoning deadlock. Previous multi-agent planners, such as in \cite{talamadupula2014coordination, buisan2021human, hwang2018dealing, liu2024event} typically maintain separate knowledge bases for each agent in the scenario. However, these static first-order representations lack the expressiveness required for more complex scenarios involving nested perspective-taking \cite{lemaignan2015mutual} and when the environment or the system changes over time. 
 
Dynamic Epistemic Logic (DEL) extends the concepts of epistemic planning and ToM by providing a formal logical framework to reason about changes in knowledge and beliefs over time~\cite{van2007dynamic}. While epistemic planning focuses on devising plans that consider the current epistemic states of agents, DEL specifically addresses how these states evolve through actions and observations~\cite{van2001games}. DEL, as a result, is a more flexible representation of the dynamics of knowledge and belief, enabling adaptive planning in scenarios where agents must continuously update their understanding of the world and each other’s mental states~\cite{ciardelli2015inquisitive}. In multi-robot systems, DEL allows each robot in the MRS to reason and plan using its beliefs of other robots' beliefs of the system while disconnected, updating its beliefs and policy if new actions or events are observed, and routing to communicate when necessary ~\cite{van2007dynamic}. DEL has recently been integrated into robotics applications. The method presented in~\cite{bolander2021based} recreates the Sally-Anne psychological test for human-robot interactions. Typical DEL-based multi-agent research uses epistemic planning for game theory-based policies~\cite{maubert2021concurrent}.

Active inference provides a probabilistic model for agents to make decisions and update beliefs by minimizing uncertainty and surprise~\cite{friston2016active}. Using Bayesian inference, agents predict sensory inputs and select actions aligned with their goals, incorporating their own and others' knowledge and beliefs. This inherently involves Theory of Mind and DEL, as agents model and anticipate others' mental states for effective interaction~\cite{pezzulo2018hierarchical}. Integrating active inference with ToM and epistemic planning enables sophisticated planning and decision-making, allowing agents to refine their understanding and actions for optimal outcomes in complex, multi-agent environments such as in~\cite{albarracin2022epistemic}. This connection is essential for developing intelligent systems capable of adaptive and cooperative behavior in uncertain and dynamic settings. Previous work on active inference in \cite{tian2020learning} utilized Policy Belief Learning to improve conveyance of intent through action, paired with a reward system, which incentivized actions that improve the overall understanding of the operational environment. Additionally, \cite{schack2024sound} evaluated how active inference could be performed in robot teams which can typically communicate and how the lack of communication can be exploited to better decrease uncertainty in an agent's beliefs of the world. Previous work was also performed in both leader-follower and leaderless models in \cite{maisto2023interactive} of which the similar approach to their leaderless work was used in our approach. Several works have incorporated realistic applications, including the authors in \cite{pezzato2023active} who use active inference and behavior trees were shown to improve the robustness of plans for a mobile manipulator. Additionally, authors in~\cite{ccatal2021robot} showed that perception, path generation, localization, and mapping naturally emerge from using active inference and minimizing free energy. 

Recently, a new perspective on the theory of mind (ToM) called the Bayesian mind has attracted attention, being suggested as a feedforward model for decision-making \cite{pezzulo2016navigating}. In this work, we extend the concept of the Bayesian mind using active inference as in \cite{priorelli2023flexible} and \cite{maisto2023interactive}, and incorporate dynamic epistemic tasking with higher-order reasoning. We demonstrate that while first-order reasoning can yield good results, higher-order reasoning provides more robust outcomes even in the absence of communication and presence of uncertain sensor measurements.

%% file: 4_problem_formulation.tex
\section{Problem Formulation}\label{sec:prob}

Consider a MRS of $n$ robots in the set $\mathcal{R}$. We let $\bm{x}_i$ denote the state variable of the robot $i$ that evolves according to general dynamics at time $t$ such that:
\begin{equation}
   \bm{\dot{x}}_i(t) = \bm{f}_i(\bm{x}_i(t),\bm{u}_i(t),\bm{\nu}_i)
   \label{eq:robot_dynamics}
\end{equation}
where $\bm{u}_i(t)\in\mathbb{R}^{d_u}$ and the variable $\bm{\nu}_i\in\mathbb{R}^{d_\nu}$ denote the control input and zero-mean Gaussian process uncertainty, respectively. The function $\bm{f}_i$ represents the stochastic dynamics of robot $i$ given a control input and process uncertainty. We also assume that all robots are equipped with sensors that allow them to ascertain certain measurements from other robots in the system. A robot's continuous observation $\bm{y}_i(t)$ at time $t$ depends on its own position $\bm{x}_i(t)$ and sensor configuration $\omega_i$ (e.g., camera, lidar) such that:
\begin{equation}
    \bm{y}_i(t) = \bm{h}_i(\bm{x}_i(t),\omega_i) + \zeta.\label{eq:observations}
\end{equation} 
where the function $\bm{h}_i$ maps a robot $i$'s position to its observation $\bm{y}_i(t)$ given noise $\zeta$.

In addition, we let the set $G_{i}\subseteq G$ represent the subset of all tasks in $G$ that robot $i$ believes is assigned to the MRS. All possible combinations of these assignments are represented by the power set $\mathcal{P}(G)$ and valid configurations for the multi-robot mission are denoted as $\mathcal{G}\subseteq\mathcal{P}(G)$. 

For ease of discussion, in this work, we assume that all robots know the location of all tasks $G$ present and the sensor configuration $\omega_i$ of each robot $i$ in the system. To methodically allocate tasks to the multi-robot system, we first assign a subset of all tasks to the multi-robot system for time $\tau$ to decrease the complexity of assigning a distinct set of tasks to each robot. We represent this problem as a bi-level resource optimization problem~\cite{huang2023bilevel}:
\begin{problem}\label{problem1}{\bf{\emph{Upper-Level -- Epistemic System Tasking:}}}
Design an epistemic strategy for an MRS to allocate a subset of tasks from $\mathcal{V}$ to the system at any given time $t$, accounting for uncertainty in local observations and considering that robots are unable to communicate.
\end{problem}

\begin{problem}\label{problem2}{\bf{\emph{Lower-Level -- Intent Signaling for Subtask Assignment:}}} 
Given the subset of tasks to complete from the upper-level optimization, formulate a policy for effective intent signaling for each robot and efficient task completion. 
\end{problem} 

The mathematical formulation of Problem~\ref{problem1} and~\ref{problem2} can be expressed as:
\begin{align}
    \min_{x, z} & \quad \sum_{\tau = 1}^{|G|} c_\tau x_\tau + \sum_{i=1}^{|\mathcal{R}|}\sum_{\tau=1}^G c^{\prime}_{i\tau} z_{i\tau} \\
\text{subject to} & \quad \sum_{\tau=1}^{|G|} x_\tau \leq K \\
& \quad x_\tau \in \{0, 1\} \ \forall \tau\in \{1,\dots,|G|\}, \\
& \quad z_{i\tau} \in \argmin_{z'_{i\tau}} \sum_{i=1}^{|\mathcal{R}|}\sum_{\tau=1}^{|G|} b_{i\tau} z^{\prime}_{i\tau}
\label{eq:overallOpt}
\end{align}
The variable $c_i$ represents the cost associated with selecting task $\tau$ represented by the binary variable $x_\tau$. $c'_{i\tau}$ represents the cost associated with assigning robot $i$ to task $\tau$ where the assignment for robot $i$ at time $\tau$ is denoted by the binary variable $z_{i\tau}$. $b_{i\tau}$ represents the cost in the lower-level problem for assigning robot $i$ to task $\tau$. $K$ is a constant representing the maximum number of tasks that can be selected from the set $G$.
The upper-level objective function minimizes the total cost of selecting tasks and assigning robots, while the lower-level problem ensures the optimal assignment of robots to the selected subset of tasks. The constraints ensure that each task is assigned to exactly one robot if selected from the set $G$ and that each robot is assigned to at most one task.

\section{Approach}\label{sec:approach}
Our proposed framework is designed for a task allocation problem (TAP) in which robots are unable to communicate explicit information, but must instead signal their intent and infer other robot's intent in the system. When solving the TAP, robots must update their beliefs about the state of the system and also empathize with what others might believe. To realize these changes, each robot observes the observable states of each robot and reasons about the system in a hierarchical manner. Initially, each robot evaluates the subtasks that need to be allocated by the system at time $t$, using epistemic reasoning to converge a common belief about the allocation of tasks within the multi-robot system. Subsequently, each robot examines the evidence related to the movements of all robots in the system and ultimately signals its own intent, employing active inference to reduce the system's free energy. The diagram in Fig.~\ref{fig:main_frame_aif} illustrates this decentralized framework, where robots first gather observations about the MRS. These observations are then filtered based on previous measurements before generating and assessing allocations for the MRS to execute at time $t$. Upon generating these solutions, the resulting perspective of robot $i$ is denoted as $s_i$ and represents the possible assignments of robots to tasks.  

\begin{figure}[ht]
    \centering
    \includegraphics[width = 0.48\textwidth]{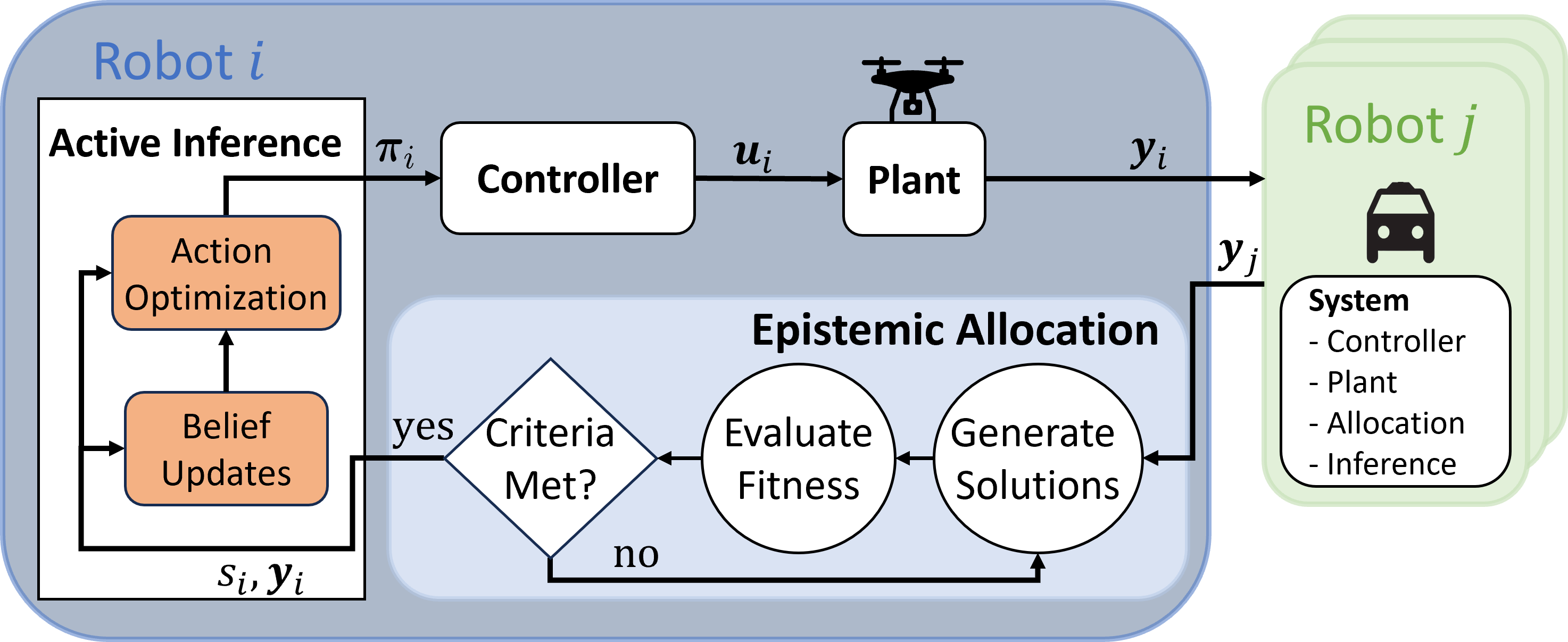}
    \caption{Diagram of the proposed approach}
    \label{fig:main_frame_aif}
\end{figure}

In the next sections, we will initially concentrate on the epistemic framing of the problem and how a robot can use higher-order reasoning to update its beliefs. We then show how active inference can be used to model how to measure the results of deeper reasoning to enhance the accuracy of non-communicative robots. We will also explore scalable runtime tools for employing this belief-based inference method before discussing the online epistemic distribution of subtasks to enable more effective reasoning during operation.

\subsection{Epistemic Structures for Multi-Robot Reasoning}
In this section, we frame the problem using dynamic epistemic logic (DEL) and epistemic planning to enable multi-robot goal selection in environments where direct communication between robots is not feasible. To overcome the challenges of coordination in this communication-limited environment, we use dynamic epistemic logic (DEL) and epistemic planning within each agent in our system. Our approach leverages DEL to model the knowledge, beliefs, and intentions of robots, integrating this with planning algorithms for goal selection and sequence optimization. To limit the possible combinations of execution policies for each robot to infer about the MRS, we use epistemic logic and allow the robots to reason about the system state. For this work, the epistemic language, $\mathcal{L}(\Psi,\AP,\mathcal{R}$)
is obtained as follows in Backus-Naur form~\cite{knuth1964backus}:
\begin{equation}
    \phi \Coloneqq H(\eta) \ | \ \phi\land\phi \ | \ \neg\phi \ | \ K_i\phi \ | \ B_i\phi  \nonumber
\end{equation}
where $i\in\mathcal{R}$. $H\in\Psi$ with $\Psi$ being a set of functions that describe the system state. $\eta$ generally denotes arguments for the functions in $\Psi$. $\neg\phi$ and $\phi\land\phi$ are propositions that can be negated and form logical conjunctions, where $\phi\in \AP$ and $\AP$ is a finite set of atomic propositions. We denote the set of possible worlds by $W$, where each world $w \in W$ represents a distinct state of the system where each robot is assigned to a subset of tasks. $K_i\phi$ and $B_i\phi$ are interpreted as ``robot $i$ knows $\phi$" and ``robot $i$ believes $\phi$", respectively. Practically, we consider $\phi$ to be the generic assignment of a robot to a task. 

We represent the global epistemic state as $s=(W,(R_i)_{i\in\mathcal{R}},L,w)$ where $L: W\rightarrow \AP $ assigns a label to each world defined by its true propositions and the accessibility relation $R_i$ represents the uncertainty of robot $i$ at run-time. An accessibility relation $R_i$ for robot $i$ defines which worlds are indistinguishable to $i$; that is, $R_i(w, v)$ holds if robot $i$ cannot distinguish between worlds $w$ and $v$. A robot may not be able to distinguish worlds if the \textit{evidence} associated with both worlds is equivalent or similar according to the uncertainty associated with our observations from \eqref{eq:observations}. 

Sequences of relations are used to represent higher-order knowledge. For example, the statement ``robot $i$ knows that robot $j$ knows $\phi$'' is true in $s$ if and only if $s\models K_i K_j\phi$. This condition is satisfied when $\phi$ is true in all worlds accessible from $w$ through the composite relation of $R_i$ and $R_j$. The perspective of robot $i$ on the system state is notated as $s_i$. 

Dynamic epistemic logic is expanded from epistemic logic through action models~\cite{bolander2021based}. These models affect a robot's perception of an event and influence its set of reachable worlds, $R_i$. A robot may plan to reduce the run-time uncertainty by taking actions. An action $a$ transforms a world $w$ to a world $w'$ such that if $w \models [a]\phi$, then $\phi$ holds in the resultant world $w'$.  Robots generate plans $\pi_i$ that are sequences of actions $\pi_i = \langle a_{i1}, a_{i2}, \ldots, a_{ik} \rangle$ leading from an initial state $s_i^0$ to a goal state $s_i^* \in \Gamma_i$. Here $\Gamma_i$ represents the set of epistemic goal states for the robot which are defined as the possible goal configurations the system aims to achieve.

\begin{equation}
\pi_i = \langle a_{i1}, a_{i2}, \ldots, a_{ip} \rangle
\end{equation}
such that
\begin{equation}
s_i^0 \xrightarrow{a_{i1}} s_{i1} \xrightarrow{a_{i2}} \cdots \xrightarrow{a_{ip}} s_i^* \in \Gamma_i.
\end{equation}

Coordination among robots is achieved through continuous observation and nested belief updates:
\begin{equation}
\bm{B}_i \leftarrow \bm{y}_i
\end{equation}
which represents what robot $i$ believes about robot $j$'s goal assignments, noting that $\bm{B}_i$ represents robot $i$'s nested beliefs about the system such that $\bm{B}_i\phi \models B_{i}\dots B_{r} \phi$ where the subscript denotes the nested belief of the $r^\text{th}$ robot from the perspective of robot $i$.
The planning process incorporates the robots' knowledge and beliefs:
\begin{equation}
\pi_i = \text{Plan}(s_i^0, \Gamma_i, \bm{B}_i)
\end{equation}
This allows robots to infer each other's goals:
\begin{equation}
\forall i, j \in \mathcal{R}, \quad B_i [\bm{y}_j] (\Gamma_j)
\end{equation}
If robot $i$ observes that robot $j$ is moving towards goal $g_m$, it updates its beliefs to $B_i (\Gamma_j = g_m)$ where if $\bm{B}_i (j \text{ is moving towards } g_m) \text{ then } \bm{B}_i (\Gamma_j = g_m)$.
Robots update their beliefs based on actions, $a$, and observations, $y_i$, using DEL:
\begin{equation}
\bm{B}_i \leftarrow a \quad \text{or} \quad \bm{B}_i \leftarrow \bm{y}_i
\end{equation}

This framework guarantees both soundness and completeness. The inference rules accurately represent the environment's state, and with enough observations, the goals can be inferred correctly.
Combining DEL with epistemic planning offers an effective strategy for managing multi-robot systems without the need for direct communication. However, refining these beliefs and developing a logical policy requires a probabilistic method that can manage complex reasoning. In the subsequent section, we merge the epistemic framework into the active inference model, utilizing nested beliefs and data gathering to signal a robot's intentions and infer the goals of other robots.

\subsection{Active Inference for Decision Making}
Active inference robots perform perception and action planning by minimizing variational free energy. To minimize free energy, these robots utilize a generative model that depicts the joint probability of the stochastic variables responsible for their perceptions~\cite{friston2017active}. Fig.~\ref{fig:gen_model_aif} shows our generative model for this framework where a robot receives observations $\bm{y}_i\in\mathcal{O}, \ \forall i\in \mathcal{R}$. Observations are then processed through generalized Bayesian filtering~\cite{friston2010generalised}, leading to an update in the beliefs of the system. Each robot can then utilize these revised beliefs to forecast the system's behavior. This process results in a robot $i$ creating a set of policies for all robots, but only able to control its own policy. However, these actions affect the environment and, in turn, the state of the system as perceived by other robots. This cycle continues, enabling the robots to infer the intentions of other robots and to use their own control policies to influence the beliefs of others. Active inference is distinct from perception or learning because it involves an active process driven by the goal of producing observations that are minimally surprising. 
\begin{figure}[ht]
    \centering
    \includegraphics[width=0.40\textwidth]{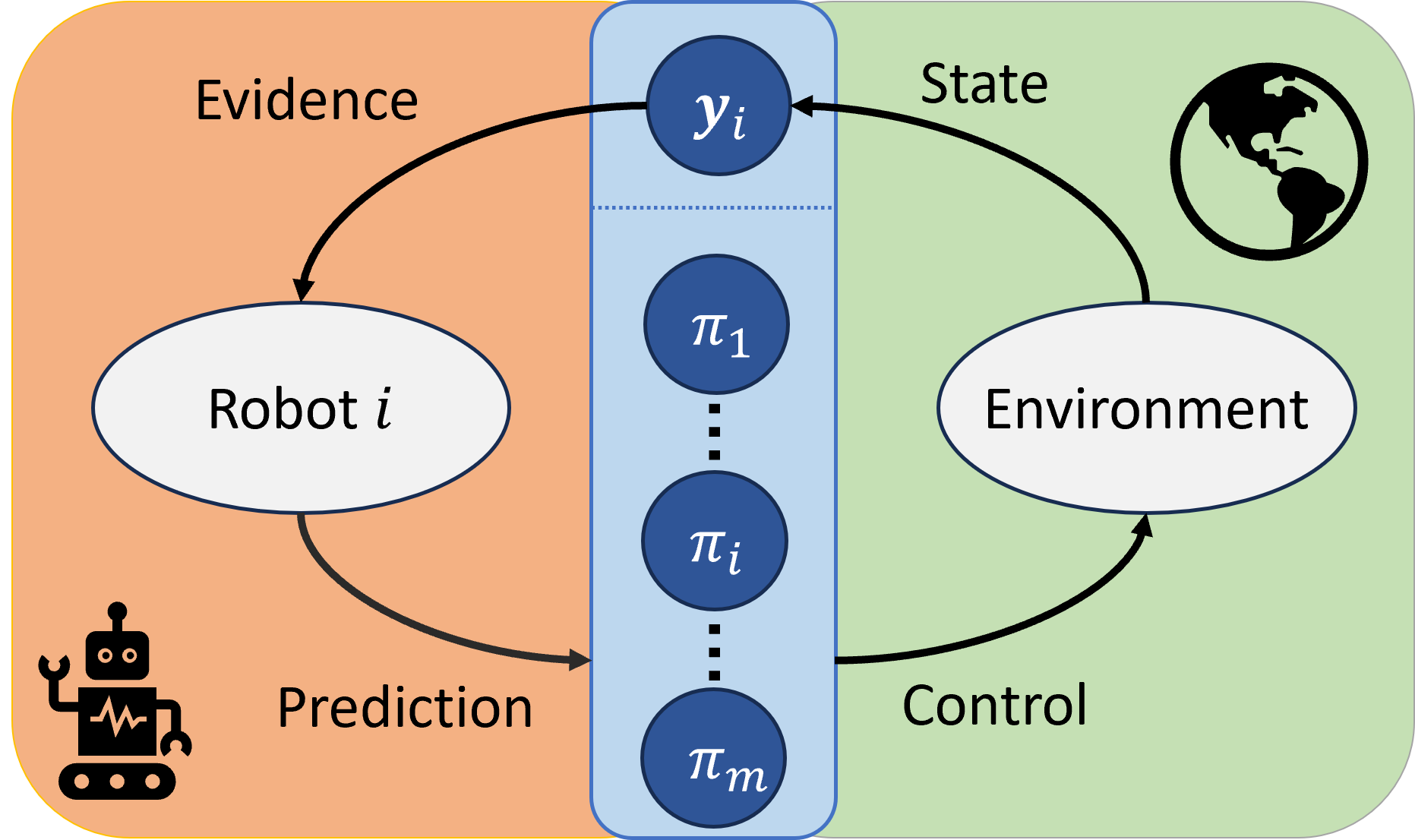}
    \caption{Overview of generative model and process used in our multi-robot application. We assume that the state is hidden and the robot is only able to observe using their own on-board sensing capability (e.g., depth sensors, cameras).}
    \label{fig:gen_model_aif}
\end{figure}

The generative model of any $i^\text{th}$ robot in our approach is mathematically defined similar to the formalism first introduced by~\cite{maisto2023interactive}; however, we augment this model with continuous states, observations, and actions represented in~\cite{parr2023generative}. We let the dynamics model for the generative model be represented by \eqref{eq:robot_dynamics}, influenced by control inputs $\bm{u}_i(t)$ and process noise. Next, we formulate the observation likelihood for a robot $i$'s observations $\bm{y}_i(t)$ at time $t$ as:
\begin{equation}
    P(\bm{y}_i(t) \mid \bm{x}_i(t), \omega_i) \sim \mathcal{N}(\bm{y}_i(t); \bm{h}_i(\bm{x}_i(t),\omega_i),\Sigma_i)
\end{equation}
where $\bm{h}_i$ maps the robot $i$'s position $\bm{x}_i(t)$ to its observation $\bm{y}_i(t)$ and $\Sigma_i$ denotes the covariance matrix that represents the effect of noise on the observation.

The two main components of active inference are belief updates and active selection. In our application, we note that each robot maintains a belief over the possible goal configurations for the multi-robot system. Depending on the application, these goal configurations should represent possible states that will accomplish a pre-defined mission. Each robot maintains this belief about the possible goal configurations that the system is performing at time $t$. We represent this posterior belief as:
\begin{equation}
    Q(\mathcal{G} \mid \bm{y}_i(t), \omega_i)
\end{equation}
where $\mathcal{G}$ is a subset of possible goal configurations that would accomplish the task allocation problem in \eqref{eq:overallOpt}. We use Bayes' rule to update the prior belief $P(\mathcal{G})$ based on the likelihood $P(\bm{y}_i(t)\mid \mathcal{G},\omega_i)$ derived from the observations and sensor configurations of robot $i$. This is modeled as follows:
\begin{equation}
    Q(\mathcal{G} \mid \bm{y}_i(t),\omega_i)\propto P(\bm{y}_i(t)\mid \mathcal{G},\omega_i)P(\mathcal{G}).
\end{equation}
These posterior updates are then used in the active selection component; however, we can only approximate the posterior given that we do not have direct access to the true system state. Therefore, we approximate the posterior $q(\mathcal{G})$ as follows:
\begin{equation}
    q(\mathcal{G}_i) \propto \bm{L}_i(\mathcal{G} \mid \bm{y}_i(t),\Omega)P(\mathcal{G})
    \label{eq:posterior_update}
\end{equation}
where robot $i$'s approximate posterior $q(\mathcal{G})$ is proportional to the product of the likelihood function $\bm{L}_i(\bm{y}_i(t),\Omega,\mathcal{G})$ and prior $P(\mathcal{G})$. The likelihood function is discussed further in the following section, but first we define how we use the belief update for the free energy calculation. 

By employing active inference, a robot executes a perception-policy loop through the application of the aforementioned matrices to hidden states and observations. In our scenario, perception involves estimating which of the valid goal configurations the system is achieving. At the start of any mission, the MRS might have access to a prior over goal configurations providing each robot with an initial state estimate, which is then refined by subsequent observations. 

For anticipated future states, the robot deduces the current hidden goal configuration $\mathcal{G}$ taking into account the expected transitions defined by the control $\bm{u}_t$ and general dynamics in \eqref{eq:robot_dynamics}. Active inference utilizes an approximate posterior for hidden states and control policies. As demonstrated by the authors in~\cite{smith2022step}, the distribution is most accurately approximated by minimizing the variational free energy (VFE), which is defined at time $t$ as:
\begin{align}
    \bm{F}(\bm{x}_i(t),&\bm{u}_i(t),\mathcal{G},\bm{y}_i(t),\omega_i) =\nonumber\\
    &H[q(\mathcal{G})] + D_{KL}[q(\mathcal{G})\parallel P(\bm{x}_i(t)\mid\bm{y}_i(t),\omega_i)]
\end{align}
where $H$ is the model uncertainty computed using Shannon entropy and $D_{KL}$ denotes the Kullback-Leibler (KL) divergence. This can be further generalized to expected free energy for a policy $\pi_i$:
\begin{align}
    \mathbb{E}_{\pi_i}[\bm{F}] = 
    &\mathbb{E}_{\pi_i}[H[q(\mathcal{G})]] + \nonumber\\
    &\mathbb{E}_{\pi_i}[D_{KL}[q(\mathcal{G}) \parallel P(\bm{x}_i(t)\mid\bm{y}_i(t),\omega_i)]]
\end{align}
The expected free energy (EFE) is a metric that integrates the entropy of the variational distribution $Q(\mathcal{G})$ with the expected log-likelihood of the generative model. By minimizing $\bm{F}$, robots adjust their beliefs to better approximate the true posterior distribution, balancing model complexity with alignment to observed data. The EFE comprises two components that assess the quality of the policy. The first component is the expected Kullbeck-Leibler divergence, which promotes low-risk policies by minimizing the discrepancy between the approximate posterior and the desired outcome. The second component is the expected entropy of the posterior over hidden states, representing the epistemic aspect of the quality score and encouraging policies that reduce uncertainty in future outcomes.

In this paper, we conceptualize the beliefs over hidden states $Q(\mathcal{G})$ for a multi-robot system (MRS) as a probabilistic framework that represents the likelihood of various goal configurations $\tilde{g} \in \mathcal{G}$ being the true state of the system. This approach allows us to encode a discrete array of probabilities that correspond to the different ways in which the robots might be arranged to achieve their respective objectives, particularly in scenarios where direct communication between robots is not feasible. 

To illustrate, consider the scenario depicted in Fig.~\ref{fig:intro_pic_aif}. In this example, there are two distinct goals: one that requires the collaboration of two robots and the other that can be accomplished by a single robot. The set of valid configurations in this scenario is given by $\mathcal{G} = [(0,1,1),(1,0,1),(1,1,0)]$. Each tuple in $\mathcal{G}$ represents a possible state of the system, where the elements of the tuple indicate the allocation of robots to each goal. For example, the configuration $(0,1,1)$ implies that one robot is assigned to the first goal, while the other two are assigned to the second goal. Thus, the belief distribution $Q(\mathcal{G})$ captures the probability that any of these configurations reflects the true state of the system. This probabilistic representation serves as a critical mechanism for decision-making, guiding robots in selecting actions that align with the most probable configurations. In doing so, the system can dynamically adapt to uncertainties and variations in the environment, optimizing the overall mission outcome despite the absence of explicit communication between robots. This method enables a form of implicit coordination in which robots rely on their shared probabilistic understanding of the mission to achieve complex objectives.

In the subsequent section, we discuss the particular likelihood function from \eqref{eq:posterior_update} used to revise a robot's belief regarding the hidden states of the system. By employing this function, robots can methodically gather evidence and deduce the actual state of the system, or, in this work, the correct goal configuration that assigns a goal to each robot.

\subsection{Higher-Order Evidence-Based Reasoning}
As previously mentioned in \eqref{eq:posterior_update},
given a set of valid goal configurations $\mathcal{G}$, a likelihood function is an important function to update a robot's approximate posterior belief about the hidden states $Q(\mathcal{G})$. A factor in interpreting likelihood based on observations is salience. Salience describes how prominent or emotionally striking something is. In neuroscience, salience is an attentional mechanism that helps organisms learn and survive by allowing them to focus on the most relevant sensory data. In our application, salience is the evidence that a robot is aligned with a goal $g_j\in G$. The set $G$ is different in that $G$ is the set of all goals in an environment, but $\mathcal{G}$ is the valid goal configurations for the multi-robot mission such that  $\mathcal{G}\subseteq\mathcal{P}(G)$ from $\mathcal{G}$.

We note that previous salience functions used in active inference and robotics literature such as in \cite{lison2010belief, maisto2023interactive}, typically only use up to first-order reasoning to define their evidence and subsequent posterior belief. We begin our formulation generally with the salience value defined as:
\begin{equation}
    \bm{e}_i^{(k)}\gets\bm{\upsilon}_i^{(k)}(\bm{y}_i,\Omega,G) = \exp\left(-\frac{1}{\eta} \bm{h}_i^{(k)}(\bm{y}_i,\Omega,G) \right)\label{eq:evidenceFunc}
\end{equation}
where the array $\bm{e}_i^{(k)}\in\R^{|G|}$ is the mapping of observations to evidence for goals in $G$ from the perspective of robot $i$ and given the sensor configurations of all robots in the system $\Omega$. The superscript $k$ denotes the level of reasoning at which the robot is evaluating its observations. Since a robot's observations are independent of other robot's observations, we can aggregate the evidence associated with each robot $i$'s perspective of other robots. For example, evidence can be evaluated as metrics such as distance to a goal or relative angles to a goal as shown in Fig.~\ref{fig:first_order_reasoning} and Fig.~\ref{fig:higher_order_reasoning}, respectively. 
Generally, the following function maps a positive evidence value to each goal configuration in $G$ such that:
\begin{equation}
    \bm{h}_i^{(k)}(\bm{y}_i,\Omega,G) = \sum_{a\in \mathcal{R}_k}\left[h_{ij,a}^{(k)}(\bm{y}_i,\Omega,g_j)\right]_{j=1}^{|G|}
    \label{eq:evidence_h}
\end{equation}
where $\mathcal{R}^k_i$ denotes the subset of robots that are considered for $k^\text{th}$-order reasoning from the perspective of robot $i$, the value $h_{ij,r}(\omega_i,\Omega,g_j)$ is positive for a goal $g_j\in G$, and robot $r$ that indicates if a robot is aligned with a goal $g_j$. The notation $|G|$ denotes the number of elements in the set $G$.

One can observe that in a heterogeneous system, a robot may not always be able to consider other robots' perspectives if the perspective is not measurable. As such, evidence $h_{ij,r}^{(k)}(\bm{y}_i,\Omega,G) = 0$ when a robot $i$ is unable to compute the evidence from the perspective of robot $r$ ($\omega_r\succ\omega_i)$. For example, consider Fig.~\ref{fig:example_reasoning}, a vehicle that might only be able to measure angles cannot abstract the depth information that other vehicles are able to observe, but angles are able to be abstracted from depth information. This information is represented in the set $\mathcal{R}_i^k\in\mathcal{R}$ and, as a result, no new information is mapped by robot $i$ from robot $r$'s perspective and is not able to inform the variational distribution for non-measurable perspectives for second- and third-order reasoning.

\begin{figure}[h]
    \centering
    \subfigure[Zero and First Order]{\fbox{\includegraphics[width=0.20\textwidth]{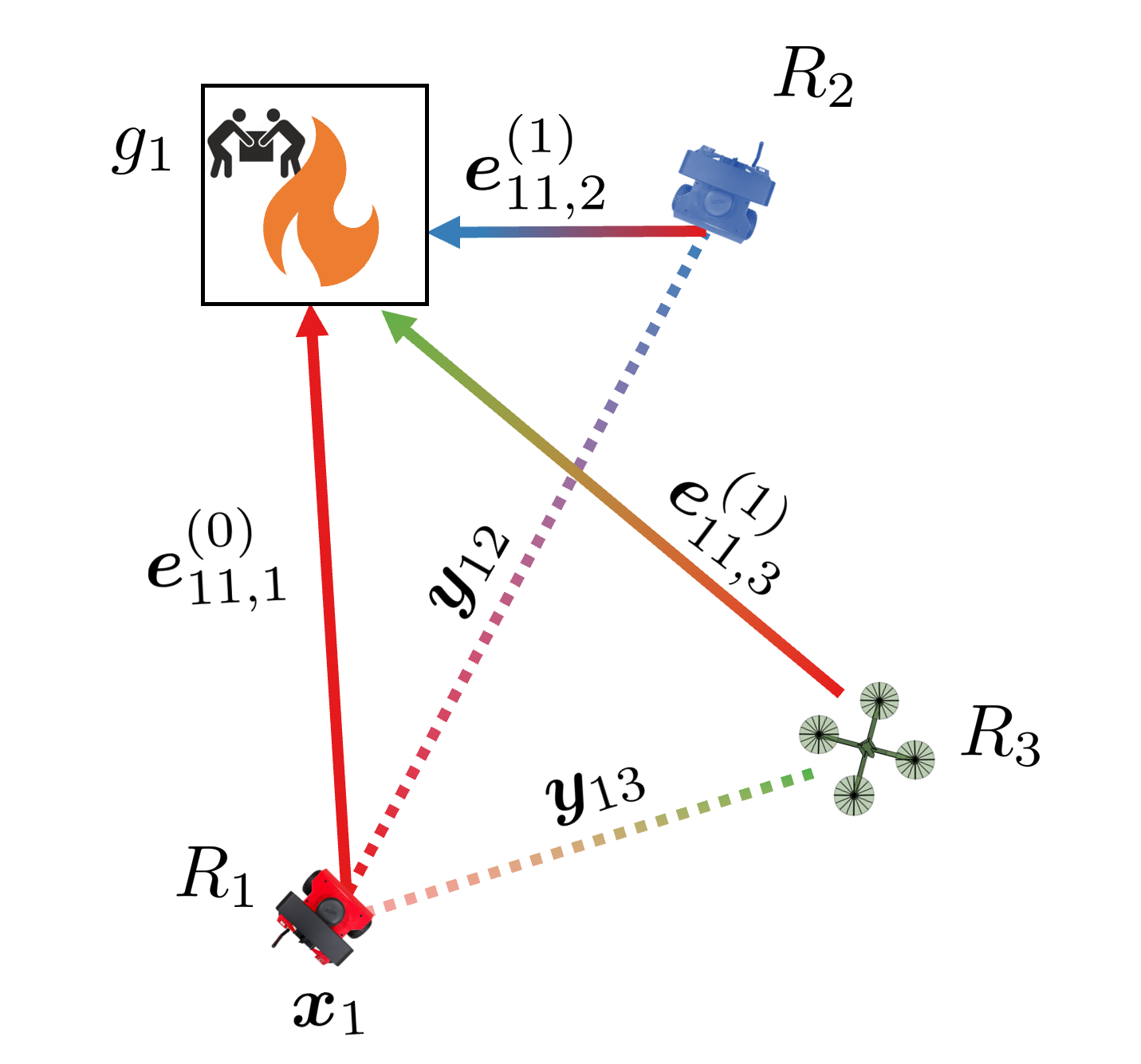}}
    \label{fig:first_order_reasoning}
    }
    \subfigure[Second and Third Order]{\fbox{\includegraphics[width=0.20\textwidth]{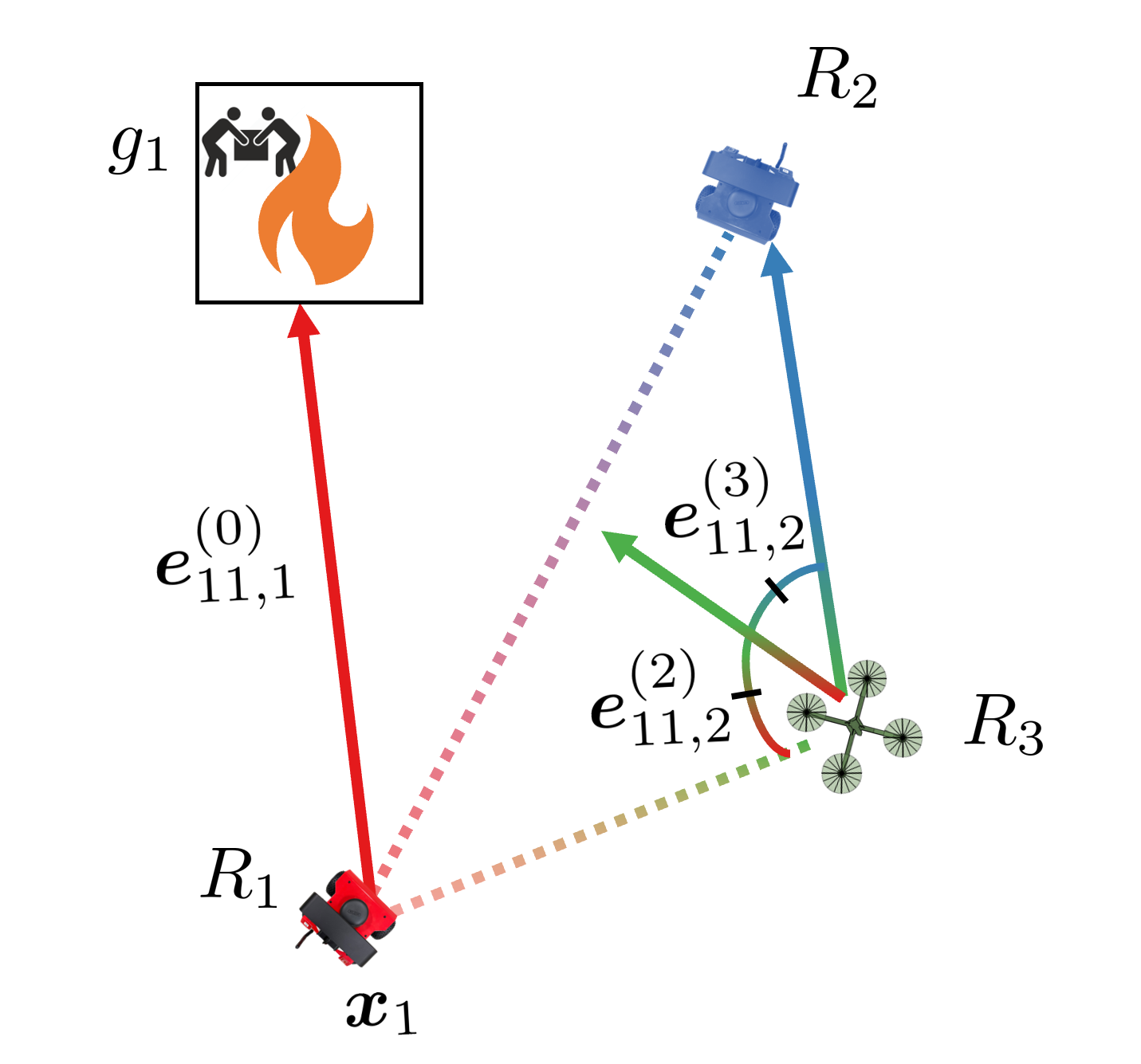}}
    \label{fig:higher_order_reasoning}
    }
    \caption{Pictorial depiction of observation mapping to evidence and depth of reasoning.}
    \label{fig:example_reasoning}
\end{figure}

Higher-order reasoning can be iteratively aggregated to form a comprehensive joint probability distribution. To calculate the likelihood of a robot being aligned with any particular goal in $G$, we let the probability distribution for $k^\text{th}$-order reasoning be represented as:
\begin{equation}
    P_i(G) = \sigma\left(\sum_k\bm{e}_i^{(k)}\right)\label{eq:higher-reasoning}
\end{equation}
where $\bm{e}_i^{(k)}$ represents the evidence gathered using \eqref{eq:evidenceFunc} and $\sigma$ is representative of the softmax function.
The belief over the goal configurations specified by $\mathcal{G}$ can be inferred using a joint probability distribution of the result from \eqref{eq:higher-reasoning}. We formulate the joint probability distribution we first initialize as:
\begin{equation}
    P_i^J(g)=1\ \forall g\in G^n 
\end{equation}
where $G^n$ represents all the possible combinations of $n$ robots assigned to $|G|$ tasks. Then we calculate the joint probability distribution for all possible goal configurations as
\begin{equation}
    P_i^J(\tilde{g}) = P_i^J(g_1,\dots,g_n) = \prod_{i = 1}^n P_i(g_i)
    \label{eq:joint_prob}
\end{equation}
where $P_i(g_i)\in P_i(G)$ and $\tilde{g}\in\mathcal{G}$. Additionally, the goals $g_1,\dots,g_n$ represents the allocation of a goal in the set $G$ to each robot. The result in \eqref{eq:joint_prob} gives the probability for all configurations in $\mathcal{G}$ and we extract and normalize the subset of valid configurations to form a distribution:
\begin{equation}
    L_i\left(\mathcal{G} \mid \bm{y}_i(t),\Omega\right) = \frac{P_i^J(\tilde{g})}{\sum_{\tilde{g}'\in \mathcal{G}} P_i^J(\tilde{g}')} \ \forall\tilde{g}\in\mathcal{G}
\end{equation}
where the likelihood $L_i\left(\mathcal{G} \mid \bm{y}_i(t),\Omega\right)$ can be used to update the prior from \eqref{eq:posterior_update}.

The joint likelihood associated with higher-order reasoning can increase the robustness of the overall system because of the integration of information across multiple layers. Suppose each independent likelihood has variance ($\sigma_i^2$). The combined variance in the joint likelihood can be lower due to the aggregation of information. In addition, joint likelihood can better manage the bias-variance tradeoff. While independent likelihoods may lead to higher variance due to lack of dependency modeling, a joint likelihood balances the bias introduced by modeling dependencies and the variance reduction due to joint estimation. The joint likelihood provides a more accurate expectation for each step, leading to more stable estimates.

\begin{lemma}
In a multi-robot system, incorporating higher-order reasoning (second- and third-order) reduces the variance of the robots' belief distributions compared to first-order reasoning by leveraging joint likelihoods that account for dependencies among robots, assuming that the errors in the observations from other robots are bounded and have a mean of zero.
\end{lemma}

\begin{proof}
We aim to prove this result by direct proof. Specifically, we will show that higher-order reasoning leads to a reduction in the variance of the belief distributions over goal configurations compared to first-order reasoning, under the assumption that the errors in the observations from other robots are bounded and have a mean of zero.

Consider a multi-robot system where each robot $i$ evaluates the evidence $\bm{e}^{(1)}_i$ for each goal $g \in G$ based on its own observations, leading to a first-order probability distribution:
\begin{equation}
    P_i^{(1)}(g) = \frac{\exp(\bm{e}^{(1)}_i(g))}{\sum_{g' \in G} \exp(\bm{e}^{(1)}_i(g'))}.
\end{equation}
This first-order distribution, $P_i^{(1)}(g)$, has a variance denoted by $\sigma_i^{2}$. However, this distribution does not incorporate the dependencies or the information that might be inferred from the perspectives of other robots.

Assume that the observations made by a robot $i$ are subject to an error term $\zeta$ from \eqref{eq:observations} and is proportional to the error of the salience measure for each pair of robots and goals, $\epsilon_{j}(g)$, which is bounded and has a mean of zero:
\begin{equation}
    \epsilon_i(g) \sim \mathcal{N}(0, \sigma_\epsilon^2) \text{ and } |\epsilon_i(g)| \leq \epsilon_{\text{max}} \text{ for all } g \in G.
\end{equation}

When higher-order reasoning is introduced, robot $i$ adjusts its belief by considering the evidence from other robots, which includes their respective error terms. For instance, in second-order reasoning, the probability distribution becomes:
\begin{equation}
    P_i^{(2)}(g) = \frac{\exp \left( \bm{e}^{(1)}_i(g) + \sum_{j \neq i} \mathbb{E}[P_j^{(1)}(g) + \epsilon_j(g)] \right)}{\sum_{g' \in G} \exp \left( \bm{e}^{(1)}_i(g') + \sum_{j \neq i} \mathbb{E}[P_j^{(1)}(g') + \epsilon_j(g')] \right)}.
\end{equation}

Because the errors $\epsilon_j(g)$ are bounded and have a mean of zero, their contribution to the variance of $P_i^{(2)}(g)$ will average out as more robots' beliefs are considered. This process reduces the overall variance in the joint probability distribution, which can be expressed as:
\begin{equation}
    P_i^J(\tilde{g}) = \prod_{g \in \tilde{g}} P_i(g).
\end{equation}

To quantify the reduction in variance, we examine the Kullback-Leibler (KL) divergence between the robots' joint beliefs and their expected beliefs:
\begin{equation}
    D_{\text{KL}}\left(P_i^J(\tilde{g}) \parallel \mathbb{E}[P_i^J(\tilde{g})]\right).
\end{equation}

Since the errors in observations are bounded and mean-zero, the KL divergence decreases as the robots' beliefs become more aligned, indicating a convergence towards a common understanding. This reduced divergence leads to a decrease in the variance of the belief distribution:
\begin{equation}
    \sigma^2_{\text{joint}} = \text{Var}\left(\prod_{g \in \tilde{g}} P_i(g)\right) < \sigma^2_{\text{first-order}}.
\end{equation}

Finally, the reduced variance is reflected in the decreased entropy of the joint belief distribution:
\begin{equation}
    H(P_i^J(\tilde{g})) = - \sum_{\tilde{g} \in \mathcal{G}} P_i^J(\tilde{g}) \log P_i^J(\tilde{g}),
\end{equation}
where $H(P_i^J(\tilde{g})) \leq H(P_i^{(1)}(\tilde{g}))$. Lower entropy indicates that the robots' beliefs are more certain and less dispersed, which corresponds to reduced variance.

In summary, by incorporating higher-order reasoning and considering bounded, mean-zero errors in observations from other robots, the system reduces the KL divergence and variance in the belief distribution, leading to more stable, accurate predictions. This proves that higher-order reasoning is beneficial for reducing uncertainty and improving the overall performance of the multi-robot system.
\end{proof}

Higher-order reasoning models, even when based on first-order measurements, reduce the overall variance of parameter estimates by capturing dependencies and interactions between different layers of reasoning. This leads to enhanced robustness, as the model can provide more stable and reliable estimates in the presence of noise and uncertainties. 

We motivate these formulations by again considering the example shown in Fig.~\ref{fig:example_reasoning}. Consider a multi-robot system consisting of two ground vehicles equipped with depth sensors and one aerial vehicle equipped with a monocular camera. The robots are attempting to allocate tasks without communication and ground robot $R_1$ is assessing evidence between a goal $g_1$ and other two robots $R_2,R_3$. In Fig.~\ref{fig:first_order_reasoning}, we show the observations and subsequent evidence calculation for $R_1$'s sensor configuration which allows $R_1$ to calculate the distance to the goal for both itself (zero-order reasoning) and other robots (first-order reasoning). In Fig.~\ref{fig:higher_order_reasoning}, since $R_1$ is equipped with a depth sensor, it is also able to abstract the sensor measurements of $R_3$, allowing $R_1$ to estimate the change in $R_3$'s belief.

We note that through minimizing the expected free energy and using the likelihood function to update the posterior we will maximize the probability of converging to a common goal configuration without any communication. However, the Bayesian method inherently suffers from the curse of dimensionality, since the number of possible joint configurations grows as the number of robots and/or the number of goals grows. Thus, in the next section, we introduce epistemic planning for reducing the goal configurations possible in each iteration.

\subsection{Epistemic Allocation for Dimensionality Reduction}

Despite the benefits of salience in managing information overload, the challenge of dimensionality remains. To address this, we use epistemic allocation, which leverages the principles of epistemic logic to reduce the set of possible configurations in $\mathcal{G}$, thereby enhancing the system's scalability and robustness. Epistemic planning involves robots making decisions based on their knowledge and beliefs, with the aim of reducing uncertainty and achieving goals in a shared environment. In this context, the function for choosing subset goals plays a crucial role. Let $B_i(g_j)$ represent robot $i$'s belief about goal $g_j$. The belief update mechanism uses a softmax function over the evidence $e_{ij}$ between robot $i$ and goal $j$, formalized as:

\begin{equation}
B_i(g_j) = \sigma\left(\sum_k \bm{e}_{ij}^{(k)}\right)
\end{equation}

This update represents the probability that robot $i$ believes goal $g_j$ is achievable, given the evidence. With higher-order reasoning, robots can collectively maximize the diversity of evidence. This involves each robot considering not just their perspective but the perspectives of all robots to select goals that maximize the collective knowledge. Formally, the set of chosen goals $G_c$ can be described as:

\begin{equation}
G_c = \{g_j \mid \arg \max_j B_i(g_j), \forall i\}
\end{equation}

This selection ensures that the chosen goals maximize the diversity and coverage of evidence across all robots, thus enhancing the collective knowledge and reducing overall uncertainty.

%% file: 5_results.tex
\section{Simulations}\label{sec:sims}
This section showcases the outcomes obtained through Python simulations of our method executed by multi-robot team. We compare zero-order reasoning~\cite{gutin2002traveling}, a first-order reasoning baseline derived from~\cite{maisto2023interactive,priorelli2023flexible}, and our method employing higher-order reasoning to reach a valid goal configuration. This section compares these levels of reasoning in three different scenarios: rendezvous, where robots converge to a common goal; task allocation, where robots converge to separate goals; and multi-task allocation, where robots individually complete sequential goals. Each scenario and its respective comparison are discussed in the following subsections.

\subsection{Rendezvous}
In the first scenario comparison, robots are tasked with converging to a single goal amongst several possible choices. Random configurations of goal locations, robot sensor configurations, and starting positions of robots are generated for 50 trials per each combination of robots and goals. The number of robots and goals varies between two and five. The size of the environment for each test is set at $30$m$\times 30$m and the maximum number of iterations or time steps per simulation is set to 150 iterations. The maximum velocity for each robot is 1m/s, and the multi-robot system has converged if all robots reach a single goal within 150 iterations and are within 1.5m of the position of the goal. The observation error is normally distributed as $\mathcal{N}(0,0.5)$ for distance measurements and $\mathcal{N}(0,0.1)$ for angular measurements. The multi-robot system is randomly spawned with one of two different types of sensor configurations. One sensor configuration is a range sensor ($\omega_1$) which can observe distance measurements to other robots, while the other configuration ($\omega_2$) can measure relative angles to the observing robot's position. We consider that $\omega_1\succ\omega_2$ since the robots are capable of abstracting angle measurements from distance measurements. We show a sample result in Fig.~\ref{fig:converge_sample_comparison} comparing first-order reasoning and higher-order reasoning in a sample environment where two robots with two different sensor configurations are trying to converge to a single goal. The red UAV can observe angles ($\omega_2$) while the blue UGV can measure distances ($\omega_1$). 
\begin{figure}[ht]
    \centering
    \subfigure[First-order reasoning]{
    \fbox{\includegraphics[width=0.22\textwidth]
    {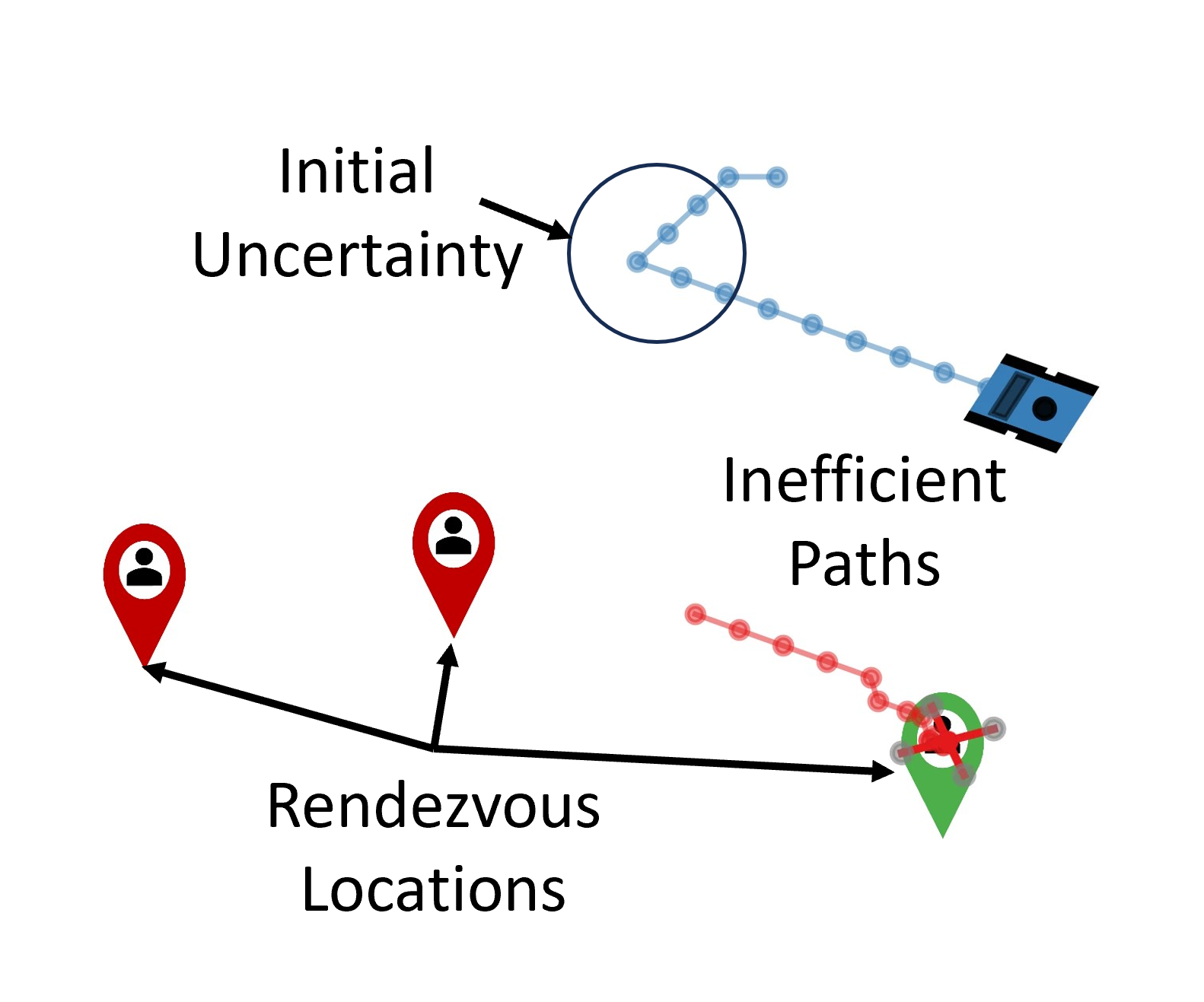}}
    \label{fig:converge_noEP}
    }%
    \subfigure[Higher-order reasoning]{
    \fbox{\includegraphics[width=0.22\textwidth]{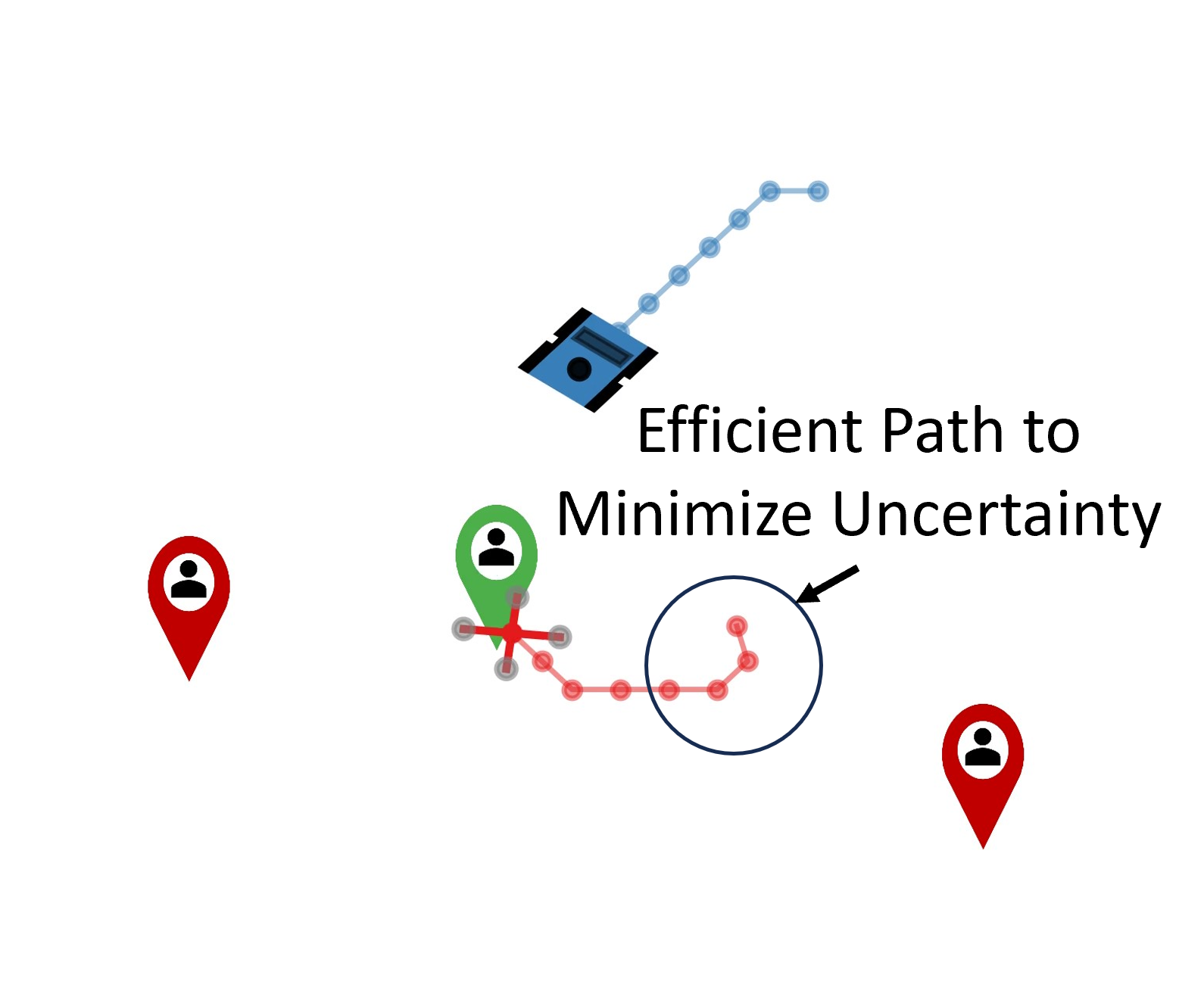}}
    \label{fig:converge_useEP}
    }
    \caption{Sample comparison of where in first-order reasoning the red UAV does not consider the blue UGV's perception of its movements}
    \label{fig:converge_sample_comparison}
\end{figure}

As shown in Fig.~\ref{fig:converge_noEP}, the red UAV does not consider the blue UGV's perception of its movements to gain more certainty about its observations before the robots end up converging to a goal farther away. In contrast, higher-order reasoning allowed the red UAV to make small movements to gain more certainty about its observations before converging to the closer goal. The results shown in this example explain how first-order reasoning is more prone to fail when the objective is to rendezvous at a goal. The results of all trials are depicted in Fig.~\ref{fig:converge_comparison_graph} which show that higher-order reasoning results in a higher success rate and that an increase in complexity does not result in a significant decrease in success, while the performance of zero and first-order reasoning continues to decrease exponentially as the number of robots increases.
\begin{figure}[ht]
    \centering
    \fbox{\includegraphics[width=0.45\textwidth]{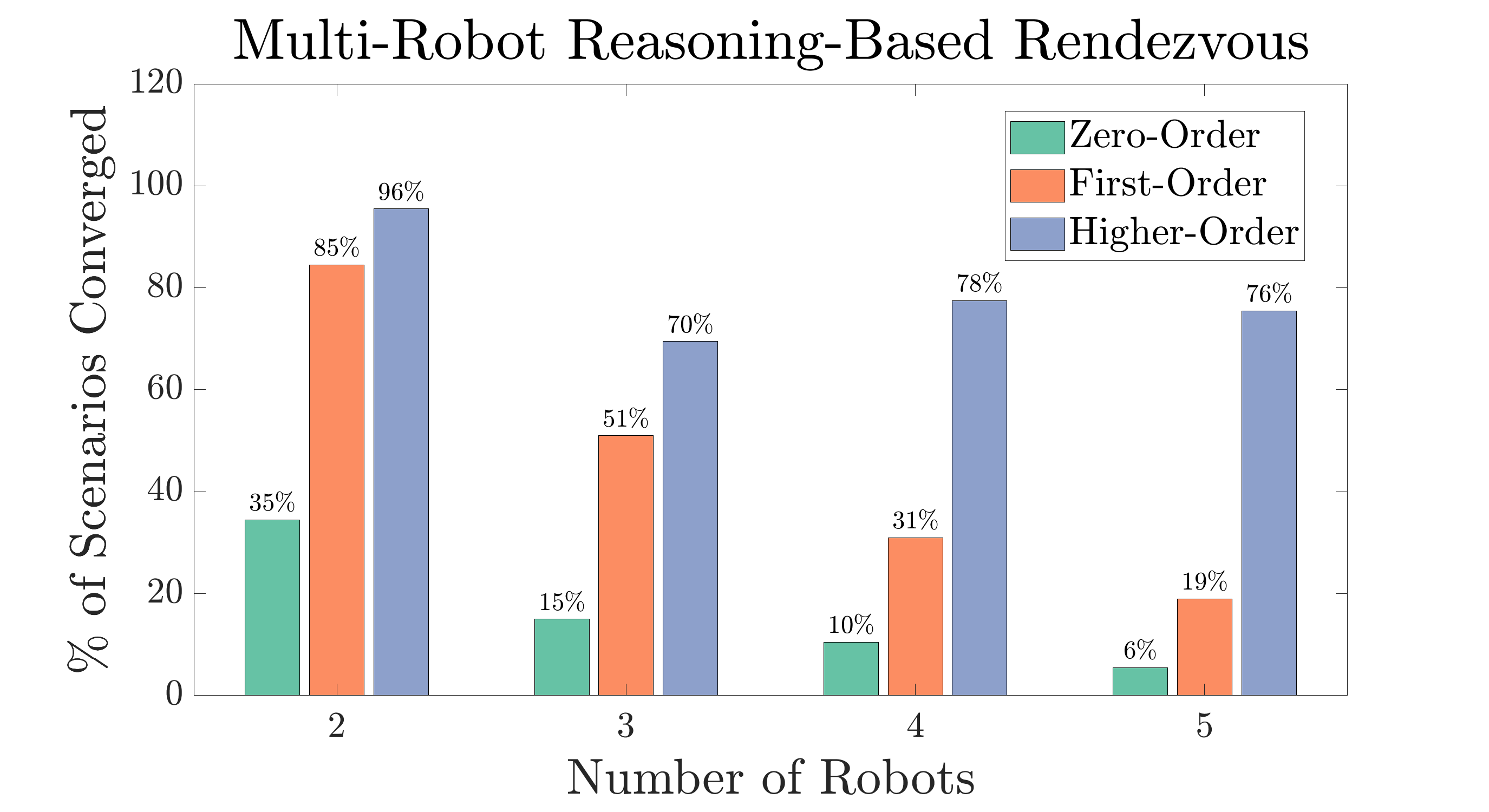}}
    \caption{Comparison of using zero- and first-order versus higher-order reasoning for a rendezvous mission.}
    \label{fig:converge_comparison_graph}
\end{figure}

\subsection{Task Allocation}
In the second set of trials, we show a comparison between the same zero- and first-order reasoning baselines and higher-order reasoning when converging to separate goals. Goal locations, robot sensor configurations, and initial robot positions are randomly generated for 50 trials per number of robots, which range from two to five robots. In these comparisons, the number of tasks is equal to the number of robots. The environment size for each trial is set at $30$m$\times 30$m and the maximum number of iterations is set to 100 iterations. The maximum velocity for each robot is 1m/s and the multi-robot system has converged if all robots reach a separate goal within 150 iterations and are within 1.5m of the position of the goal. Observation error is normally distributed as $\mathcal{N}(0,0.5)$ for distance measurements and $\mathcal{N}(0,0.1)$ for angular measurements. The multi-robot system is randomly spawned with one of two different types of sensor configurations as in previous sections. We show a sample result in Fig.~\ref{fig:allocation_sample_comparison} comparing first-order reasoning and higher-order reasoning in a sample environment where two robots with two different sensor configurations are trying to converge to a single goal.

\begin{figure}[ht]
    \centering
    \subfigure[First-order reasoning]{
    \fbox{\includegraphics[width=0.22\textwidth]
    {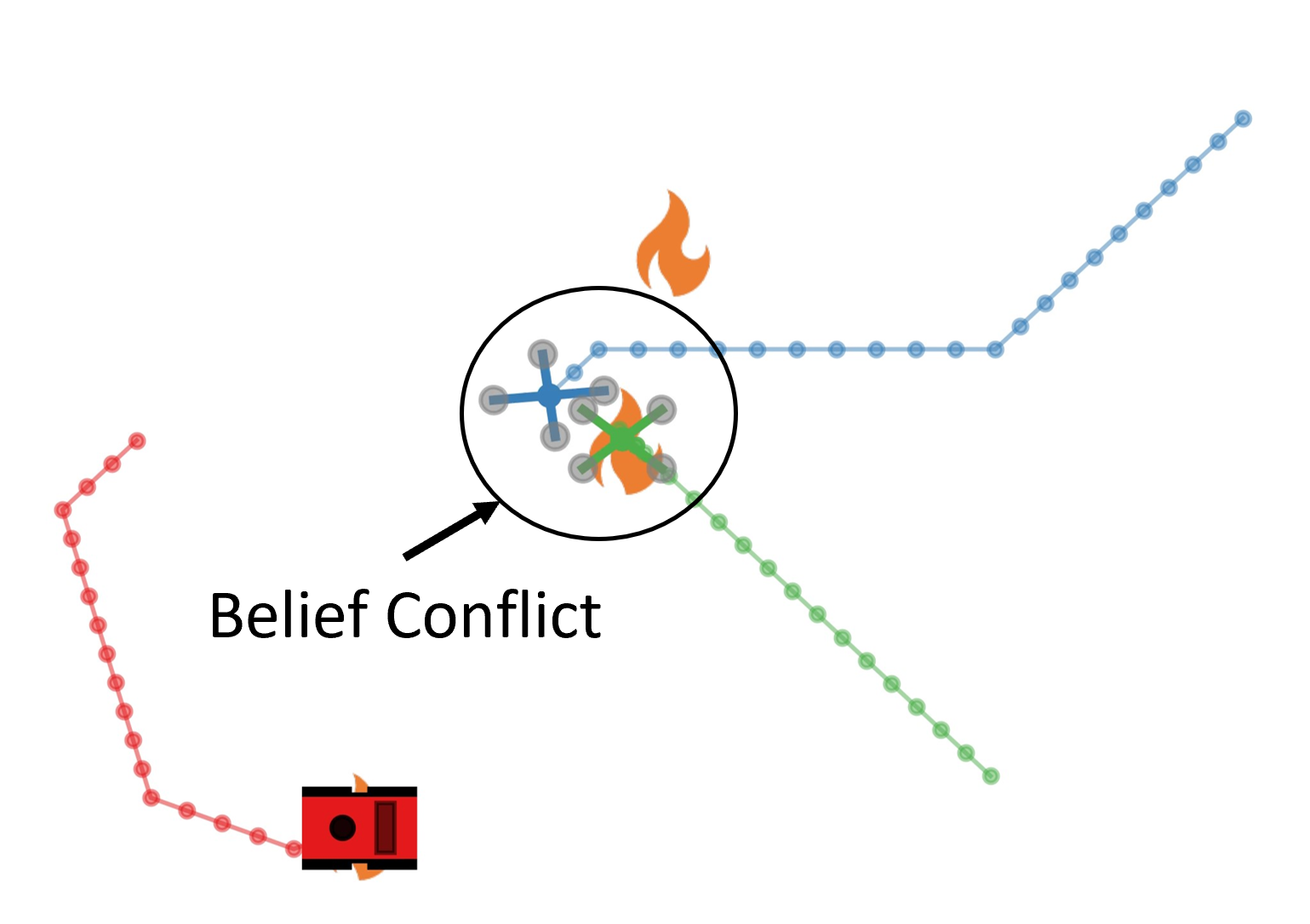}}
    \label{fig:diverge_noEP}
    }%
    \subfigure[Higher-order reasoning]{
    \fbox{\includegraphics[width=0.22\textwidth]{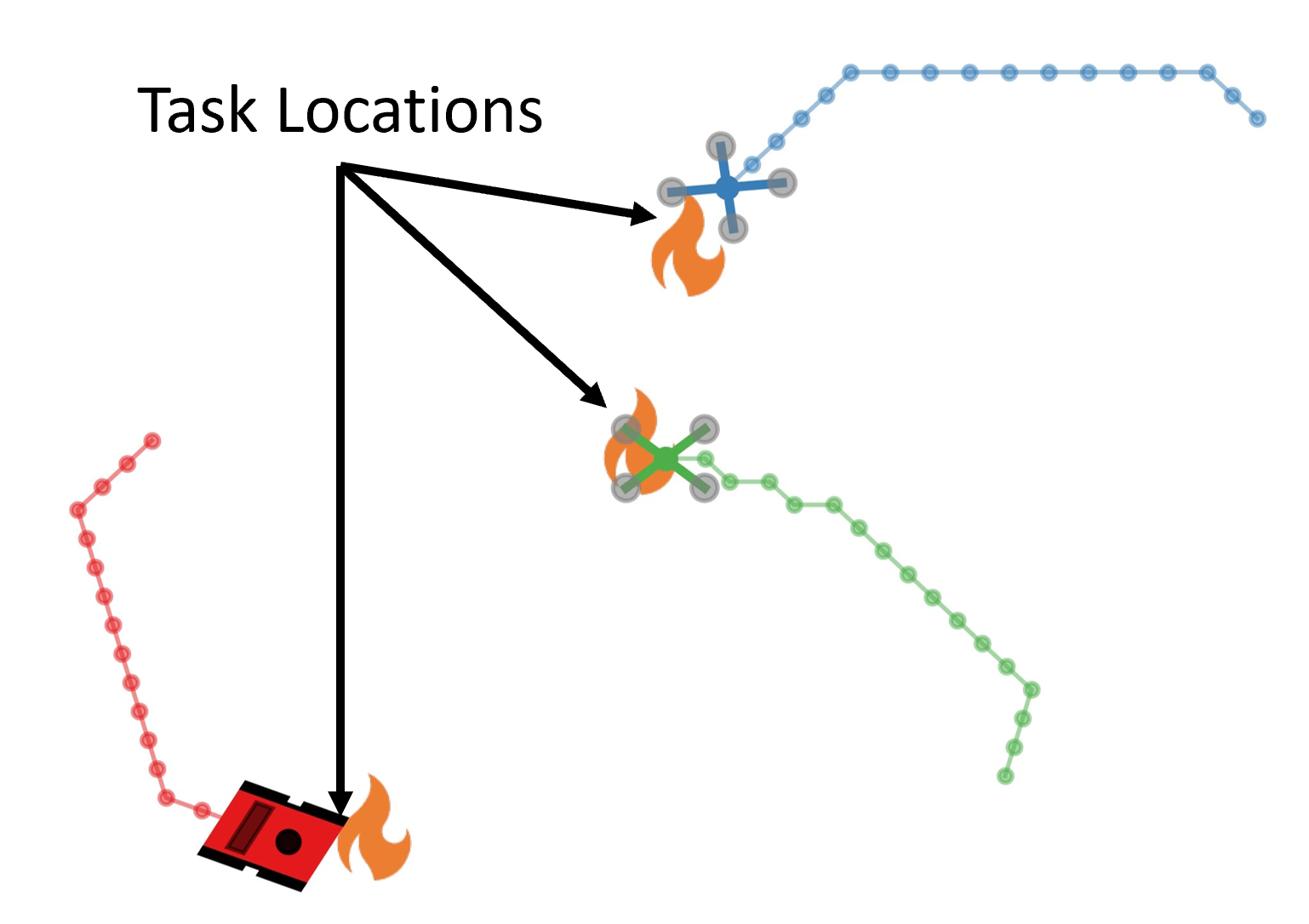}}
    \label{fig:diverge_useEP}
    }
    \caption{Illustration of our simulations showing that first-order reasoning fails to allow the blue and green UAVs to convey their intentions or reason from the other robot's perspective, resulting in both robots converging on the same task. Higher-order reasoning allows the UAVs to interpret and signal clear intentions, successfully completing different tasks.}
    \label{fig:allocation_sample_comparison}
\end{figure}

As illustrated in Fig.~\ref{fig:diverge_noEP}, the blue and green UAVs are incapable of resolving their belief discrepancies. In contrast, Fig.~\ref{fig:diverge_useEP} demonstrates that through higher-order reasoning, the UAVs can convey and understand intentions, allowing them to resolve task assignment conflicts without communication. The results of all trials are depicted in Fig.~\ref{fig:diverge_comparison_graph} which show that higher-order reasoning improves the success rate and that an increase in complexity does not result in a significant decrease in success similar to the results from Fig.~\ref{fig:converge_comparison_graph}.

\begin{figure}[h]
    \centering
    \fbox{\includegraphics[width=0.45\textwidth]{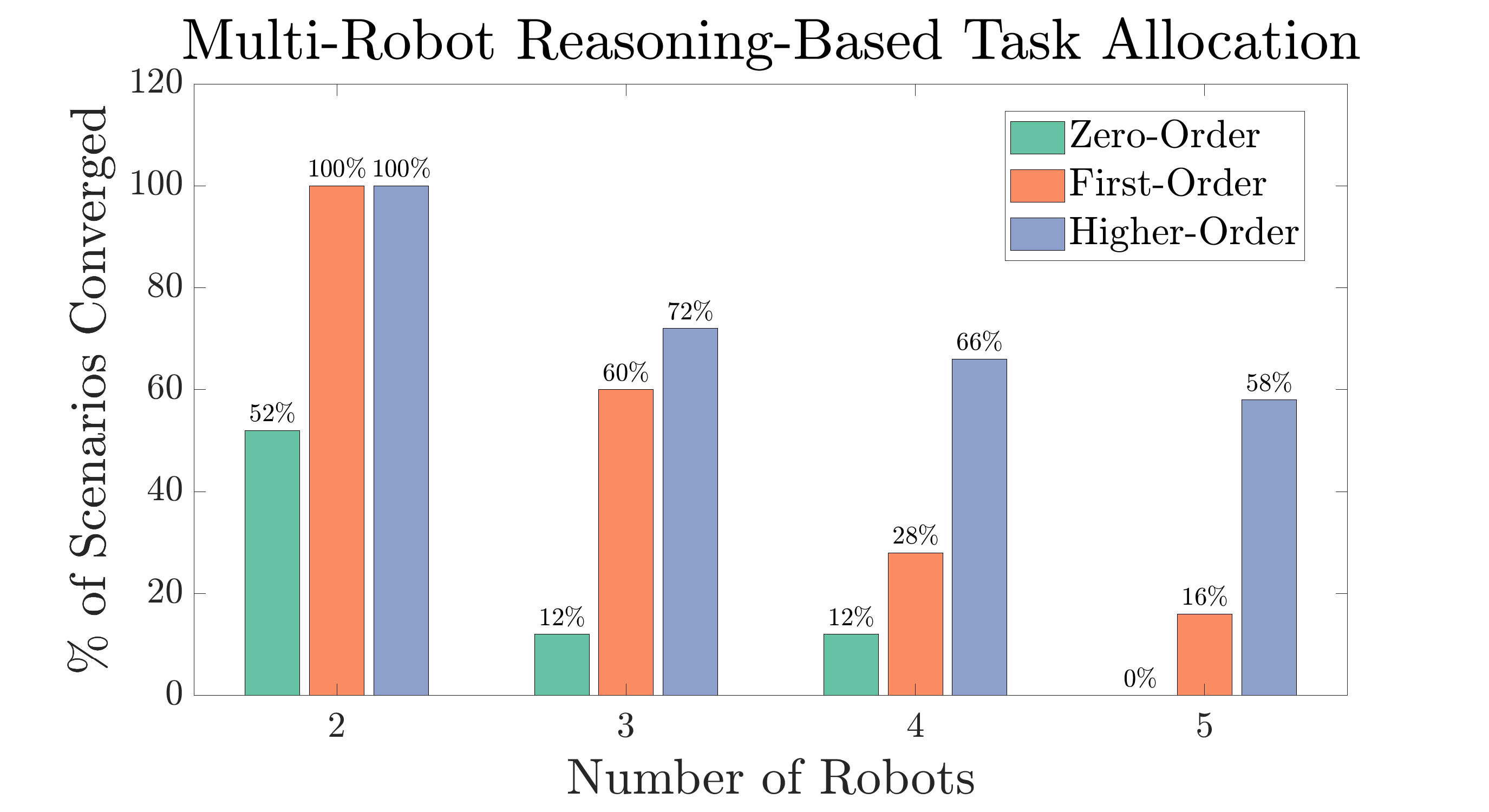}}
    \caption{Comparison of using zero- and first-order versus higher-order reasoning for a task allocation mission.}
    \label{fig:diverge_comparison_graph}
\end{figure}

\subsection{Multi-Robot Multi-Task Allocation}
Lastly, we show our method's performance for a multi-robot multi-task assignment problem where robots must decide without explicit communication or a centralized algorithm to accomplish tasks in the environment. Target locations, robot sensor configurations, initial robot positions, and task locations are randomly generated for 30 trials per number of robots ranging between 2 and 8 robots, as well as number of tasks ranging between 20 and 45 tasks. In total, the data for each level of reasoning and the number of robots is aggregated for 150 trials per category. In Fig.~\ref{fig:iterative_comparison_graph}, we show that utilizing higher-order reasoning allows robots to decrease redundancy when completing tasks in the environment and accomplish all tasks more quickly than just zero- or first-order reasoning. We observe that while the performance difference is minimal for teams of two robots, significant improvements are evident as the team size increases, especially for groups with 4 or more robots. The following example in Fig.~\ref{fig:multi_allocation_sample_comparison} illustrates that employing lower-order reasoning prevents robots from understanding the intentions of others, leading to duplicated tasks and extending the time required to accomplish tasks in the environment. In contrast, higher-order reasoning enables robots to communicate their intentions, thereby reducing redundancy and increasing system efficiency.
\begin{figure}[h]
    \centering
    \fbox{\includegraphics[width=0.45\textwidth]{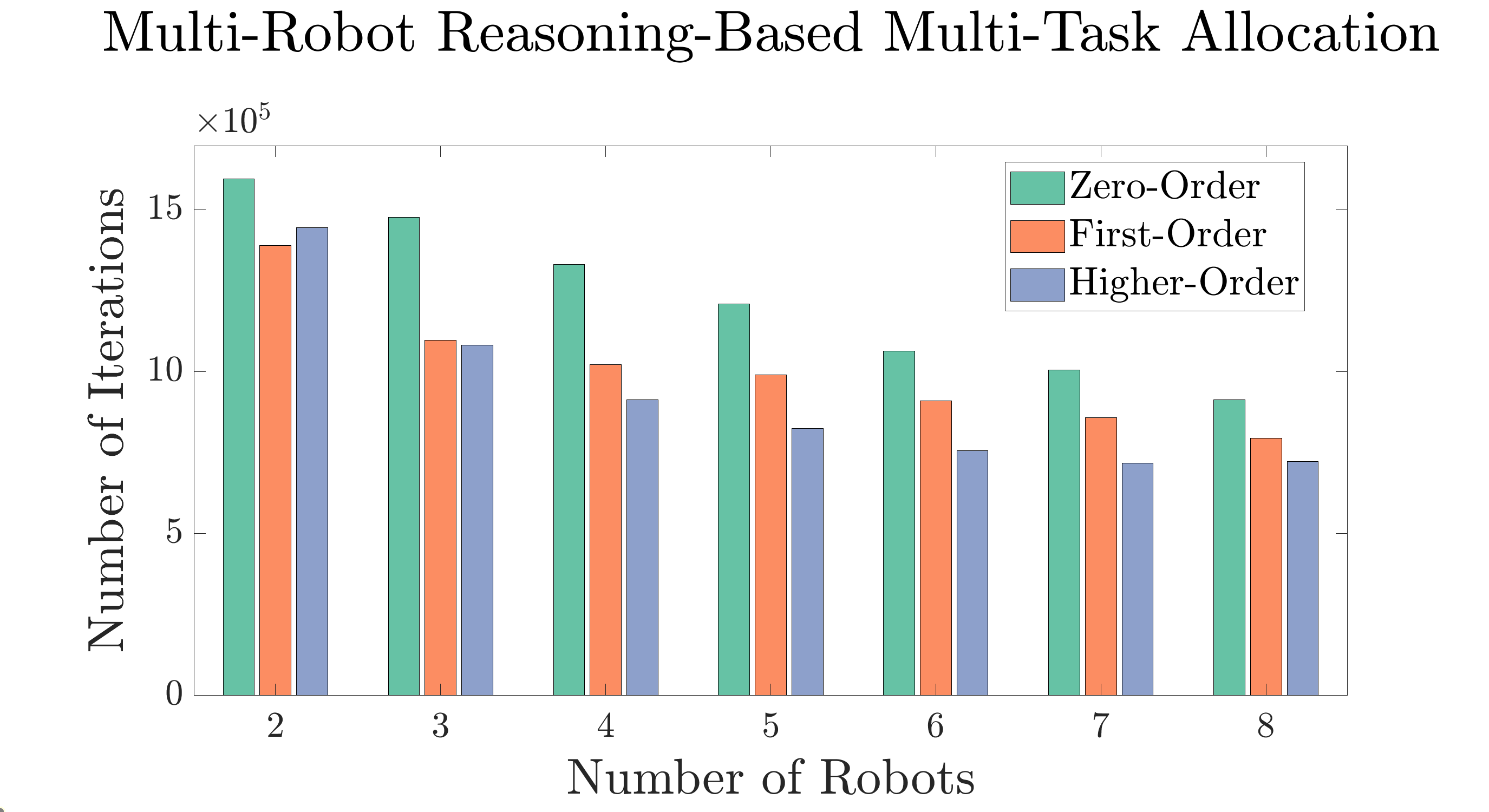}}
    \caption{Comparison of using zero- and first-order versus higher-order reasoning for a task allocation mission.}
    \label{fig:iterative_comparison_graph}
\end{figure}

\begin{figure}[ht]
    \subfigure[First-order reasoning]{
    \fbox{\includegraphics[width=0.223\textwidth]
    {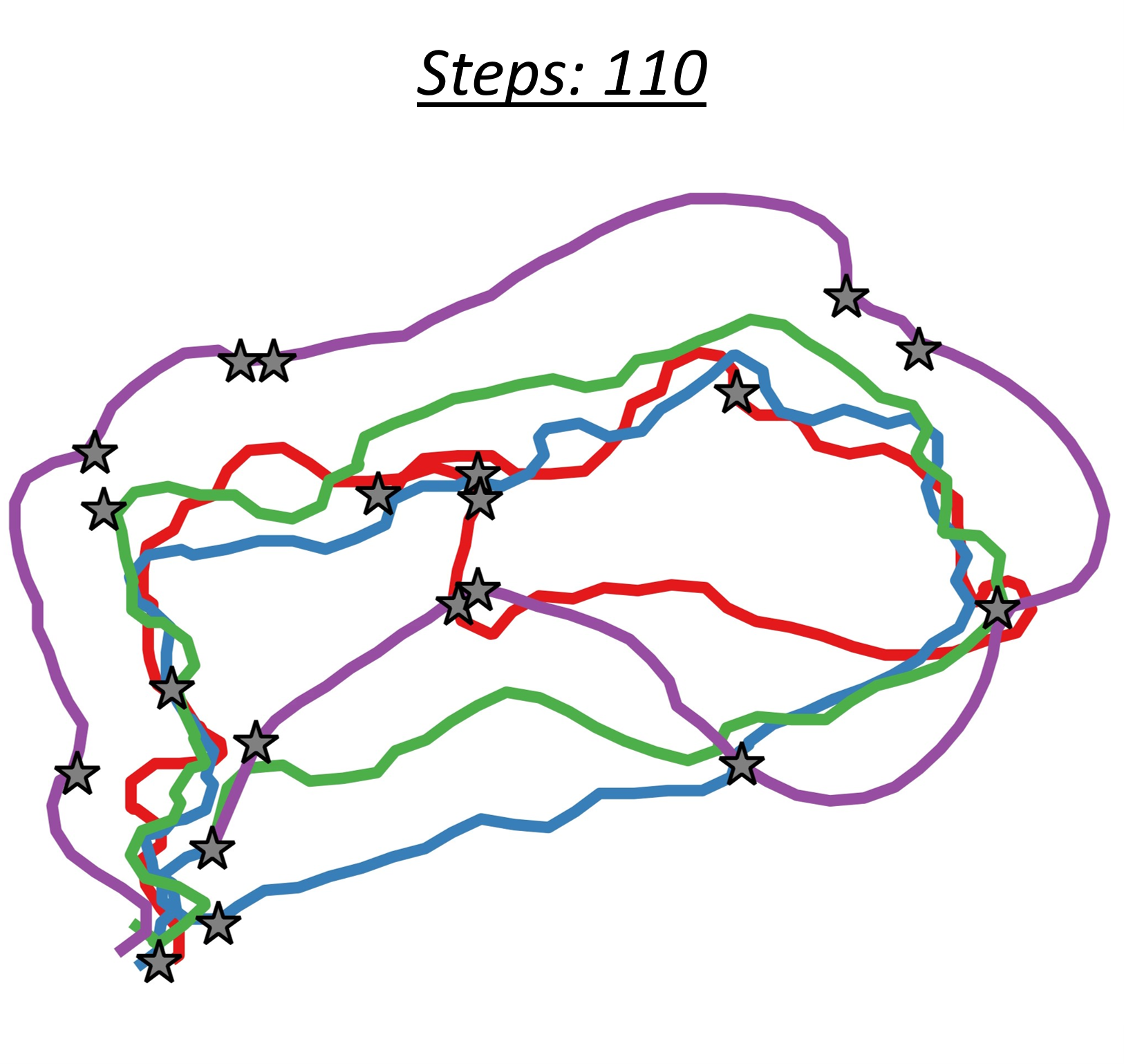}}
    \label{fig:lower_order_iterative}
    }%
    \subfigure[Higher-order reasoning]{\hspace{-0.4em}
    \fbox{\includegraphics[width=0.223\textwidth]{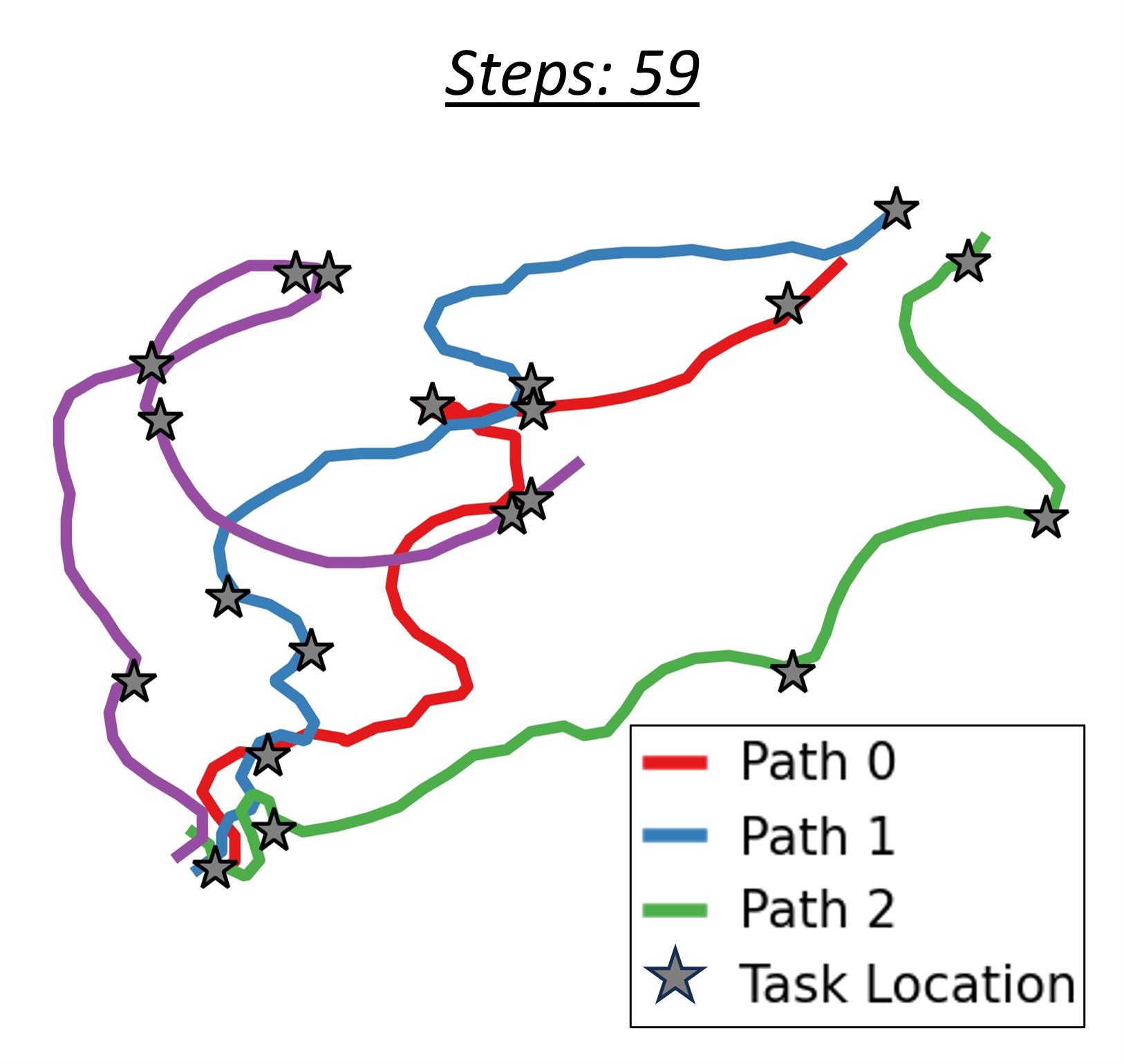}}
    \label{fig:higher_order_iterative}
    }
    \caption{Illustration of our simulations for multi-robot multi-task scenarios. In (a), the robots are unable to discern each others' intentions and causes redundant task completion. In (b), the robots are able to identify distinct tasks to accomplish and decipher other robots' intentions.}
    \label{fig:multi_allocation_sample_comparison}
\end{figure}

\section{Experiments}\label{sec:experiments}
Our approach was also validated through several laboratory experiments with a multi-robot team. The team consists of several Husarion ROSbot 2.0s that used a Vicon motion capture system for localization. Vehicles start at various positions in the environment. The experiments were carried out in a $4$m$ \times 5.5$m space. The results of a sample experiment with two potential equidistant rendezvous locations and three ground vehicles are shown in Fig.~\ref{fig:experiment_rz}.
\begin{figure}[h]
    \centering
    \includegraphics[width=0.45\textwidth]{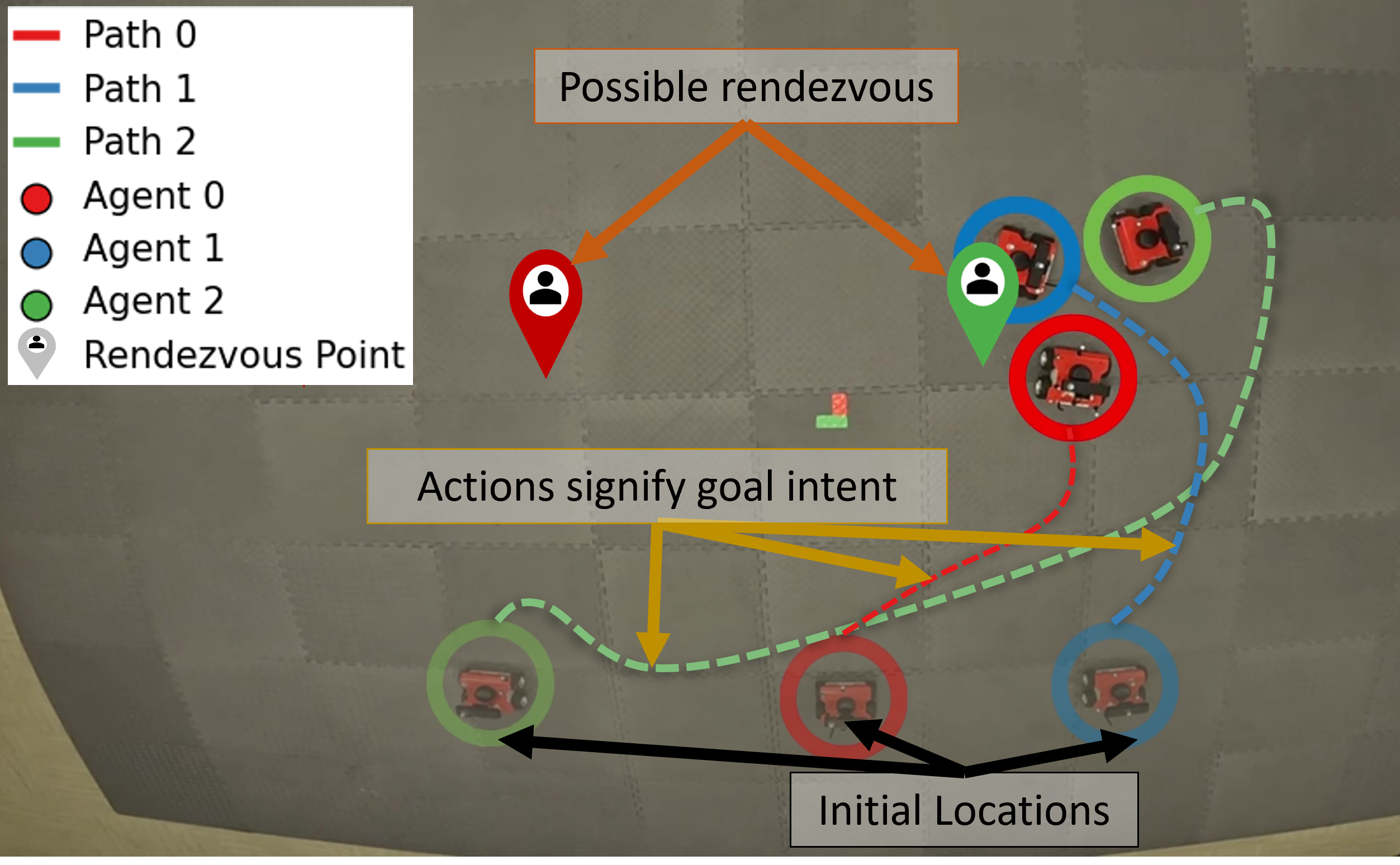}
    \caption{Experiment where three robots rendezvous at one of two equidistant locations without explicit communication.}
    \label{fig:experiment_rz}
\end{figure}

As shown in the figure, each robot initially is uncertain about which goal the system should converge to. After a small number of measurements, the robots makes their intentions explicit by minimizing free energy and maximizing evidence that they are moving toward the green rendezvous point. The robots accomplish this by taking exaggerated paths toward the green rendezvous point. This is also the case depicted in Fig.~\ref{fig:experiment_task} where robots move toward distinct goals and signal their intentions using exaggerated paths.

\begin{figure}[h]
    \centering
    \includegraphics[width=0.45\textwidth]{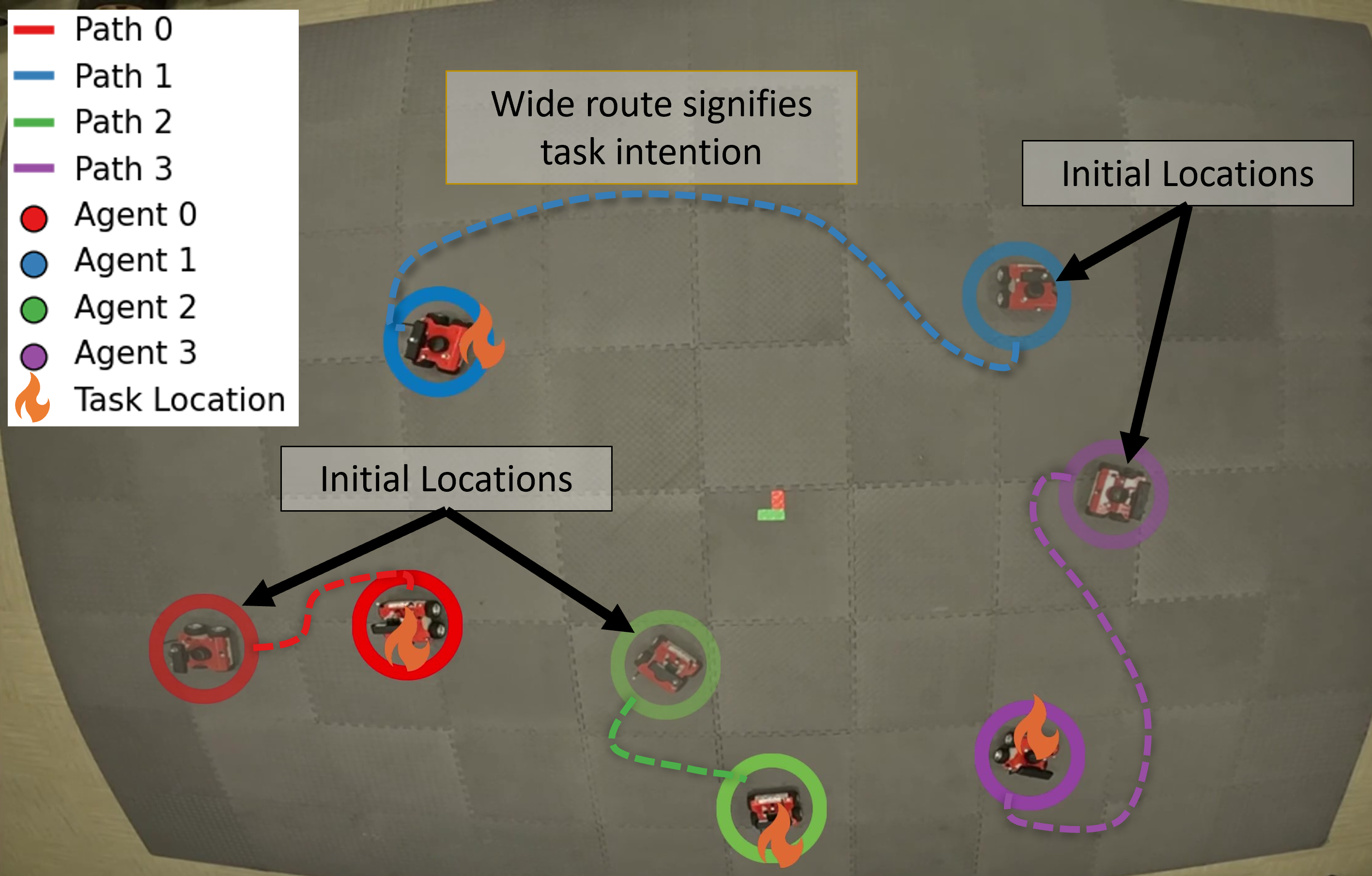}
    \caption{A four robot experiment where each robot needs to accomplish a distinct task without explicit communication.}
    \label{fig:experiment_task}
\end{figure}

We show similarly that we can perform multiple tasks per robot with heterogeneous sensing capabilities. Fig.~\ref{fig:iterative_sample_comparison} shows several snapshots of the results of this sample virtual experiment using the RotorS Firefly and Clearpath Jackal models in Gazebo and RViz where an aerial vehicle can only observe the angles of other robots from tasks, while the ground robots can observe the depth. Similarly to the simulations, we assume that the ground robots can abstract the angle measurements of the aerial vehicle. Fig.~\ref{fig:iterative_1} shows the starting location of all the robots and the tasks. Fig.~\ref{fig:iterative_2} shows that the robots initially move to signify which tasks they are going to complete while Fig.~\ref{fig:iterative_3} shows the majority of tasks accomplished without explicit communication. In Fig.~\ref{fig:iterative_4} all tasks have been accomplished and the robots return to their initial location. In this way, the robots cooperatively complete all tasks in the environment, accounting for the heterogeneity of the robots in the system and without explicit communication.
\begin{figure}[h]
    \centering
    \subfigure[]{
    \includegraphics[width=0.225\textwidth]
    {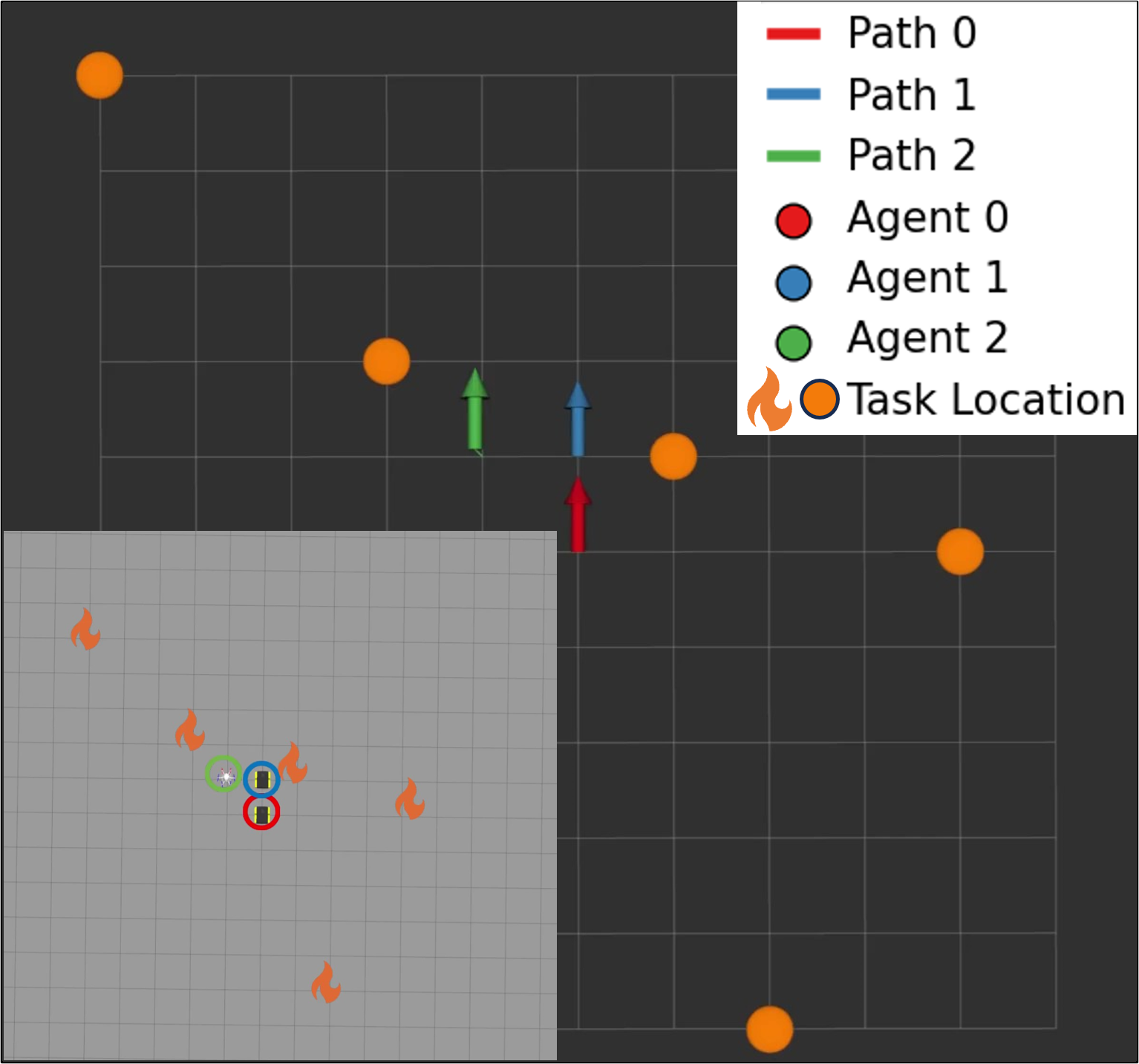}
    \label{fig:iterative_1}
    }%
    \subfigure[]{
    \includegraphics[width=0.225\textwidth]{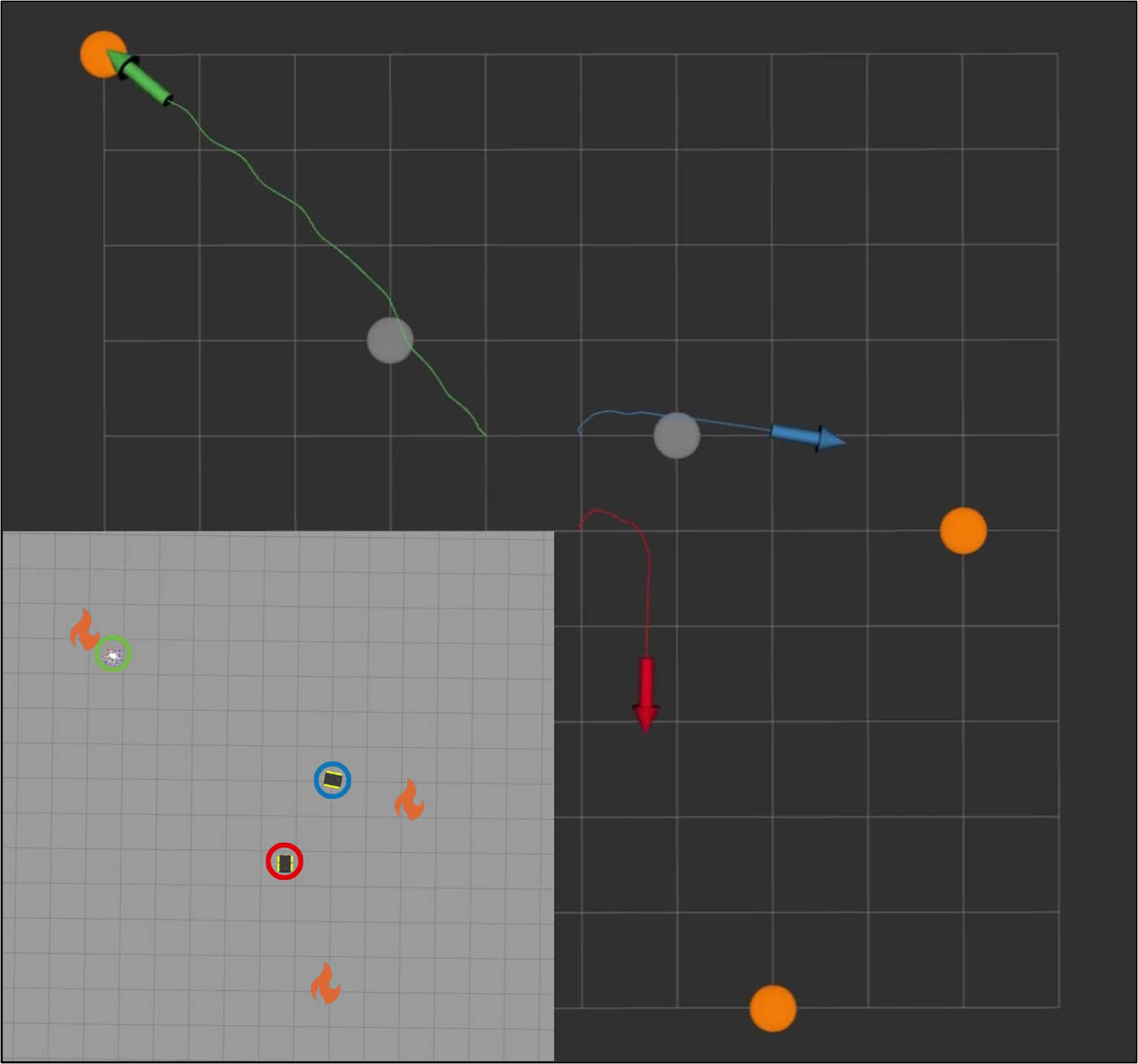}
    \label{fig:iterative_2}
    }\\
    \subfigure[]{
    \includegraphics[width=0.225\textwidth]{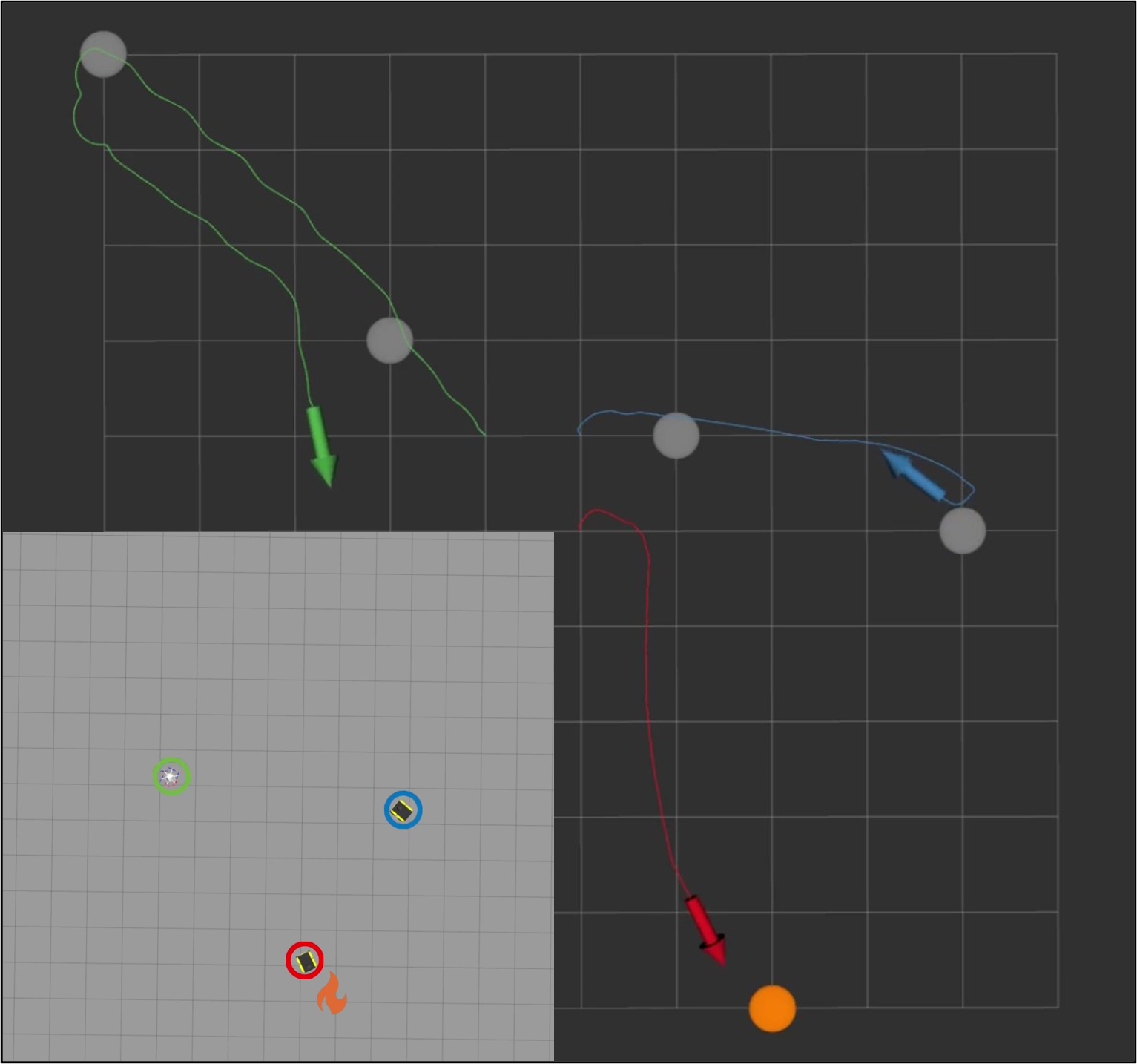}
    \label{fig:iterative_3}
    }%
    \subfigure[]{
    \includegraphics[width=0.225\textwidth]{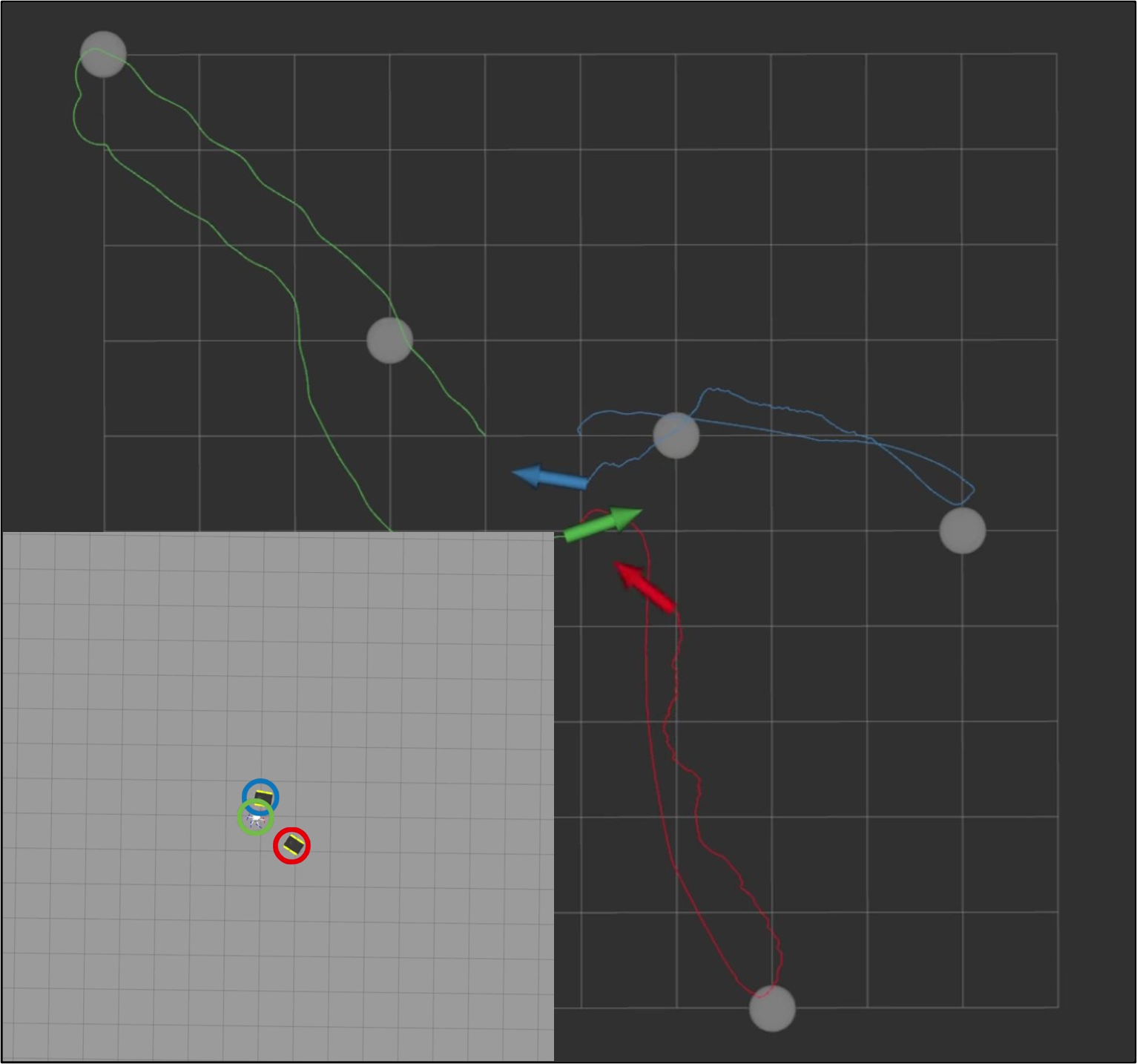}
    \label{fig:iterative_4}
    }
    \caption{RViz and Gazebo snapshots of virtual experiment with heterogeneous vehicles. As shown, the robots use higher-order reasoning to accomplish multiple tasks, even with different sensor configurations.}
    \label{fig:iterative_sample_comparison}
\end{figure}

%% file: 6_conclusion.tex
\section{Conclusion}\label{sec:conclusion}
In this work, we demonstrated the effectiveness of utilizing higher-order reasoning for multi-robot systems (MRS) operating under communication constraints. By integrating theory of mind (ToM) and epistemic planning, our proposed framework allows robots to infer the knowledge and intentions of others based on their observations and last known states. This approach enables robots to cooperate and achieve common goals even when explicit communication is not possible.

Our findings show that higher-order reasoning, extending up to the third level, significantly enhances the ability of MRS to converge to correct belief states and complete tasks efficiently. The hierarchical epistemic planning combined with active inference for runtime plan adaptation provides a robust solution to mitigate the challenges of limited communication in heterogeneous robot teams. Future work will focus on optimizing epistemic planning techniques and exploring the integration of even higher levels of reasoning. Additionally, we aim to extend our framework to more complex scenarios and larger robot teams, such as sensors with limited field-of-views and cluttered environments, further enhancing the decision-making capabilities of multi-agent systems.

%% file: 7_acks.tex
\section{Acknowledgements}
This work is based on research sponsored by Northrop Grumman through the University Basic Research Program.

%% file: 0_main.bbl
\begin{thebibliography}{10}
\providecommand{\url}[1]{#1}
\csname url@samestyle\endcsname
\providecommand{\newblock}{\relax}
\providecommand{\bibinfo}[2]{#2}
\providecommand{\BIBentrySTDinterwordspacing}{\spaceskip=0pt\relax}
\providecommand{\BIBentryALTinterwordstretchfactor}{4}
\providecommand{\BIBentryALTinterwordspacing}{\spaceskip=\fontdimen2\font plus
\BIBentryALTinterwordstretchfactor\fontdimen3\font minus \fontdimen4\font\relax}
\providecommand{\BIBforeignlanguage}[2]{{%
\expandafter\ifx\csname l@#1\endcsname\relax
\typeout{** WARNING: IEEEtran.bst: No hyphenation pattern has been}%
\typeout{** loaded for the language `#1'. Using the pattern for}%
\typeout{** the default language instead.}%
\else
\language=\csname l@#1\endcsname
\fi
#2}}
\providecommand{\BIBdecl}{\relax}
\BIBdecl

\bibitem{valle2015theory}
A.~Valle, D.~Massaro, I.~Castelli, and A.~Marchetti, ``Theory of mind development in adolescence and early adulthood: The growing complexity of recursive thinking ability,'' \emph{Europe's journal of psychology}, vol.~11, no.~1, p. 112, 2015.

\bibitem{bolander2011epistemic}
T.~Bolander and M.~B. Andersen, ``Epistemic planning for single-and multi-agent systems,'' \emph{Journal of Applied Non-Classical Logics}, vol.~21, no.~1, pp. 9--34, 2011.

\bibitem{bramblett2023epi}
L.~Bramblett, S.~Gao, and N.~Bezzo, ``Epistemic prediction and planning with implicit coordination for multi-robot teams in communication restricted environments,'' in \emph{2023 IEEE International Conference on Robotics and Automation (ICRA)}, 2023, pp. 5744--5750.

\bibitem{bramblett2023frontiers}
L.~Bramblett and N.~Bezzo, ``Epistemic planning for multi-robot systems in communication-restricted environments,'' \emph{Frontiers in Robotics and AI}, vol.~10, p. 1149439, 2023.

\bibitem{maisto2023interactive}
D.~Maisto, F.~Donnarumma, and G.~Pezzulo, ``Interactive inference: a multi-agent model of cooperative joint actions,'' \emph{IEEE Transactions on Systems, Man, and Cybernetics: Systems}, 2023.

\bibitem{bramblett2024robust}
L.~Bramblett, B.~Miloradovic, P.~Sherman, A.~V. Papadopoulos, and N.~Bezzo, ``Robust online epistemic replanning of multi-robot missions,'' \emph{arXiv preprint arXiv:2403.00641}, 2024.

\bibitem{ho2022planning}
M.~K. Ho, R.~Saxe, and F.~Cushman, ``Planning with theory of mind,'' \emph{Trends in Cognitive Sciences}, vol.~26, no.~11, pp. 959--971, 2022.

\bibitem{anderson2004integrated}
J.~R. Anderson, D.~Bothell, M.~D. Byrne, S.~Douglass, C.~Lebiere, and Y.~Qin, ``An integrated theory of the mind.'' \emph{Psychological review}, vol. 111, no.~4, p. 1036, 2004.

\bibitem{muise2022efficient}
C.~Muise, V.~Belle, P.~Felli, S.~McIlraith, T.~Miller, A.~R. Pearce, and L.~Sonenberg, ``Efficient multi-agent epistemic planning: Teaching planners about nested belief,'' \emph{Artificial Intelligence}, vol. 302, p. 103605, 2022.

\bibitem{zhang2023adaptation}
Y.~Zhang and B.~Williams, ``Adaptation and communication in human-robot teaming to handle discrepancies in agents’ beliefs about plans,'' in \emph{Proceedings of the International Conference on Automated Planning and Scheduling}, vol.~33, no.~1, 2023, pp. 462--471.

\bibitem{engesser2017cooperative}
T.~Engesser, T.~Bolander, R.~Mattm{\"u}ller, and B.~Nebel, ``Cooperative epistemic multi-agent planning for implicit coordination,'' \emph{arXiv preprint arXiv:1703.02196}, 2017.

\bibitem{talamadupula2014coordination}
K.~Talamadupula, G.~Briggs, T.~Chakraborti, M.~Scheutz, and S.~Kambhampati, ``Coordination in human-robot teams using mental modeling and plan recognition,'' in \emph{2014 IEEE/RSJ International Conference on Intelligent Robots and Systems}.\hskip 1em plus 0.5em minus 0.4em\relax IEEE, 2014, pp. 2957--2962.

\bibitem{buisan2021human}
G.~Buisan and R.~Alami, ``A human-aware task planner explicitly reasoning about human and robot decision, action and reaction,'' in \emph{Companion of the 2021 ACM/IEEE International Conference on Human-Robot Interaction}, 2021, pp. 544--548.

\bibitem{hwang2018dealing}
J.~Hwang, J.~Kim, A.~Ahmadi, M.~Choi, and J.~Tani, ``Dealing with large-scale spatio-temporal patterns in imitative interaction between a robot and a human by using the predictive coding framework,'' \emph{IEEE Transactions on Systems, Man, and Cybernetics: Systems}, vol.~50, no.~5, pp. 1918--1931, 2018.

\bibitem{liu2024event}
Y.~Liu, X.~Xie, J.~Sun, and D.~Yang, ``Event-triggered privacy preservation consensus control and containment control for nonlinear mass: An output mask approach,'' \emph{IEEE Transactions on Systems, Man, and Cybernetics: Systems}, 2024.

\bibitem{lemaignan2015mutual}
S.~Lemaignan and P.~Dillenbourg, ``Mutual modelling in robotics: Inspirations for the next steps,'' in \emph{Proceedings of the Tenth Annual ACM/IEEE International Conference on Human-Robot Interaction}, 2015, pp. 303--310.

\bibitem{van2007dynamic}
H.~Van~Ditmarsch, W.~van Der~Hoek, and B.~Kooi, \emph{Dynamic epistemic logic}.\hskip 1em plus 0.5em minus 0.4em\relax Springer Science \& Business Media, 2007, vol. 337.

\bibitem{van2001games}
J.~Van~Benthem, ``Games in dynamic-epistemic logic,'' \emph{Bulletin of Economic Research}, vol.~53, no.~4, pp. 219--248, 2001.

\bibitem{ciardelli2015inquisitive}
I.~A. Ciardelli and F.~Roelofsen, ``Inquisitive dynamic epistemic logic,'' \emph{Synthese}, vol. 192, no.~6, pp. 1643--1687, 2015.

\bibitem{bolander2021based}
T.~Bolander, L.~Dissing, and N.~Herrmann, ``Del-based epistemic planning for human-robot collaboration: Theory and implementation,'' in \emph{Proceedings of the International Conference on Principles of Knowledge Representation and Reasoning}, vol.~18, no.~1, 2021, pp. 120--129.

\bibitem{maubert2021concurrent}
B.~Maubert, S.~Pinchinat, F.~Schwarzentruber, and S.~Stranieri, ``Concurrent games in dynamic epistemic logic,'' in \emph{Proceedings of the Twenty-Ninth International Joint Conference on Artificial Intelligence}, 2021, pp. 1877--1883.

\bibitem{friston2016active}
K.~Friston, T.~FitzGerald, F.~Rigoli, P.~Schwartenbeck, G.~Pezzulo \emph{et~al.}, ``Active inference and learning,'' \emph{Neuroscience \& Biobehavioral Reviews}, vol.~68, pp. 862--879, 2016.

\bibitem{pezzulo2018hierarchical}
G.~Pezzulo, F.~Rigoli, and K.~J. Friston, ``Hierarchical active inference: a theory of motivated control,'' \emph{Trends in cognitive sciences}, vol.~22, no.~4, pp. 294--306, 2018.

\bibitem{albarracin2022epistemic}
M.~Albarracin, D.~Demekas, M.~J. Ramstead, and C.~Heins, ``Epistemic communities under active inference,'' \emph{Entropy}, vol.~24, no.~4, p. 476, 2022.

\bibitem{tian2020learning}
Z.~Tian, S.~Zou, I.~Davies, T.~Warr, L.~Wu, H.~B. Ammar, and J.~Wang, ``Learning to communicate implicitly by actions,'' in \emph{Proceedings of the AAAI Conference on Artificial Intelligence}, vol.~34, no.~05, 2020, pp. 7261--7268.

\bibitem{schack2024sound}
M.~A. Schack, J.~G. Rogers, and N.~T. Dantam, ``The sound of silence: Exploiting information from the lack of communication,'' \emph{IEEE Robotics and Automation Letters}, 2024.

\bibitem{pezzato2023active}
C.~Pezzato, C.~H. Corbato, S.~Bonhof, and M.~Wisse, ``Active inference and behavior trees for reactive action planning and execution in robotics,'' \emph{IEEE Transactions on Robotics}, vol.~39, no.~2, pp. 1050--1069, 2023.

\bibitem{ccatal2021robot}
O.~{\c{C}}atal, T.~Verbelen, T.~Van~de Maele, B.~Dhoedt, and A.~Safron, ``Robot navigation as hierarchical active inference,'' \emph{Neural Networks}, vol. 142, pp. 192--204, 2021.

\bibitem{pezzulo2016navigating}
G.~Pezzulo and P.~Cisek, ``Navigating the affordance landscape: feedback control as a process model of behavior and cognition,'' \emph{Trends in cognitive sciences}, vol.~20, no.~6, pp. 414--424, 2016.

\bibitem{priorelli2023flexible}
M.~Priorelli and I.~P. Stoianov, ``Flexible intentions: An active inference theory,'' \emph{Frontiers in Computational Neuroscience}, vol.~17, p. 1128694, 2023.

\bibitem{huang2023bilevel}
P.-Q. Huang, Q.~Zhang, and Y.~Wang, ``Bilevel optimization via collaborations among lower-level optimization tasks,'' \emph{IEEE Transactions on Evolutionary Computation}, 2023.

\bibitem{knuth1964backus}
D.~E. Knuth, ``Backus normal form vs. backus naur form,'' \emph{Communications of the ACM}, vol.~7, no.~12, pp. 735--736, 1964.

\bibitem{friston2017active}
K.~Friston, T.~FitzGerald, F.~Rigoli, P.~Schwartenbeck, and G.~Pezzulo, ``Active inference: a process theory,'' \emph{Neural computation}, vol.~29, no.~1, pp. 1--49, 2017.

\bibitem{friston2010generalised}
K.~Friston, K.~Stephan, B.~Li, and J.~Daunizeau, ``Generalised filtering.'' \emph{Mathematical Problems in Engineering}, vol. 2010, 2010.

\bibitem{parr2023generative}
T.~Parr, K.~Friston, and G.~Pezzulo, ``Generative models for sequential dynamics in active inference,'' \emph{Cognitive Neurodynamics}, pp. 1--14, 2023.

\bibitem{smith2022step}
R.~Smith, K.~J. Friston, and C.~J. Whyte, ``A step-by-step tutorial on active inference and its application to empirical data,'' \emph{Journal of mathematical psychology}, vol. 107, p. 102632, 2022.

\bibitem{lison2010belief}
P.~Lison, C.~Ehrler, and G.-J.~M. Kruijff, ``Belief modelling for situation awareness in human-robot interaction,'' in \emph{19th International Symposium in Robot and Human Interactive Communication}.\hskip 1em plus 0.5em minus 0.4em\relax IEEE, 2010, pp. 138--143.

\bibitem{gutin2002traveling}
G.~Gutin, A.~Yeo, and A.~Zverovich, ``Traveling salesman should not be greedy: domination analysis of greedy-type heuristics for the tsp,'' \emph{Discrete Applied Mathematics}, vol. 117, no. 1-3, pp. 81--86, 2002.

\end{thebibliography}
